\documentclass{article}

\usepackage{microtype}
\usepackage{graphicx}
\usepackage{subfigure}
\usepackage{booktabs} 

\usepackage{hyperref}

\usepackage[accepted]{icml2024}

\usepackage{amsmath}
\usepackage{amssymb}
\usepackage{mathtools}
\usepackage{amsthm}
\usepackage{stmaryrd}
\usepackage{bm} 

\usepackage[capitalize,noabbrev]{cleveref}

\theoremstyle{plain}
\newtheorem{theorem}{Theorem}[section]

\newtheorem{lemma}[theorem]{Lemma}
\newtheorem{corollary}[theorem]{Corollary}
\theoremstyle{definition}
\newtheorem{definition}[theorem]{Definition}

\theoremstyle{remark}
\newtheorem{remark}[theorem]{Remark}

\newtheorem*{theorem*}{Theorem}
\newtheorem*{lemma*}{Lemma}
\newtheorem{fact}[theorem]{Fact}
\newtheorem*{fact*}{Fact}
\newtheorem*{proposition*}{Proposition}
\newtheorem*{corollary*}{Corollary}

\newtheorem*{hypothesis*}{Hypothesis}

\newtheorem*{conjecture*}{Conjecture}
\theoremstyle{definition}
\newtheorem*{definition*}{Definition}

\newtheorem*{construction*}{Construction}

\newtheorem*{example*}{Example}

\newtheorem*{question*}{Question}
\newtheorem*{assumption*}{Assumption}

\newtheorem*{problem*}{Problem}
\newtheorem{model}[theorem]{Model}
\newtheorem*{model*}{Model}
\theoremstyle{remark}

\newtheorem*{claim*}{Claim}
\newtheorem*{remark*}{Remark}

\newtheorem*{observation*}{Observation}

\usepackage[textsize=tiny]{todonotes}

\newcommand{\Authornote}[2]{{\sffamily\small\color{red}{[#1: #2]}}}

\newcommand{\Rnote}{\Authornote{Rajai}}

\usepackage{boxedminipage}

\newcommand{\paren}[1]{(#1)}
\newcommand{\Paren}[1]{\left(#1\right)}

\newcommand{\brac}[1]{[#1]}
\newcommand{\Brac}[1]{\left[#1\right]}

\newcommand{\bracbb}[1]{\llbracket#1\rrbracket}
\newcommand{\Bracbb}[1]{\left\llbracket#1\right\rrbracket}

\newcommand{\Abs}[1]{\left\lvert#1\right\rvert}

\newcommand{\Card}[1]{\left\lvert#1\right\rvert}

\newcommand{\set}[1]{\{#1\}}
\newcommand{\Set}[1]{\left\{#1\right\}}

\newcommand{\Norm}[1]{\left\lVert#1\right\rVert}

\newcommand{\Snorm}[1]{\Norm{#1}^2}

\newcommand{\Normo}[1]{\Norm{#1}_1}

\newcommand{\iprod}[1]{\langle#1\rangle}

\newcommand{\Esymb}{\mathbb{E}}

\DeclareMathOperator*{\E}{\Esymb}

\newcommand{\given}{\mathrel{}\middle\vert\mathrel{}}
\newcommand{\suchthat}{\;\middle\vert\;}

\newcommand\bdot\bullet

\newcommand{\N}{\mathbb N}
\newcommand{\R}{\mathbb R}

\newcommand{\cE}{\mathcal E}

\newcommand{\cI}{\mathcal I}

\newcommand{\cN}{\mathcal N}

\newcommand{\cR}{\mathcal R}

\newcommand{\cT}{\mathcal T}

\newcommand{\bbP}{\mathbb P}

\renewcommand{\leq}{\leqslant}

\renewcommand{\geq}{\geqslant}
\renewcommand{\ge}{\geqslant}
\let\epsilon=\varepsilon

\newcommand{\eps}{\epsilon}

\allowdisplaybreaks
\newcommand*{\Normf}[1]{\Norm{#1}_{\mathrm{F}}}

\newcommand*{\transpose}[1]{{#1}{}^{\mkern-1.5mu\mathsf{T}}}
\newcommand*{\dyad}[1]{#1#1{}^{\mkern-1.5mu\mathsf{T}}}

\renewcommand{\ij}{{ij}}

\newcommand{\SBM}{\mathsf{SBM}}
\newcommand{\sbm}{\SBM_{n, 2, d,\eps}}
\newcommand{\sbmx}[1]{\SBM_{2,d,\eps}(#1)}
\newcommand{\sbmk}{\SBM_{n, k, d,\eps}}

\newcommand{\sbmkx}[1]{\SBM_{k,d,\eps}(#1)}

\newcommand{\hsbm}{\textnormal{($\cT$,k,t)-}\mathsf{MV\textnormal{-}SBM}_{n}}
\newcommand{\hsbmt}{\textnormal{(d,$\eps$,k,t)-}\mathsf{MV\textnormal{-}SBM}_{n}}

\newcommand{\Cbar}{\overline{C}}

\icmltitlerunning{Multi-View Stochastic Block Models}

\begin{document}
\onecolumn
\icmltitle{Multi-View Stochastic Block Models}



\icmlsetsymbol{equal}{*}

\begin{icmlauthorlist}
\icmlauthor{Vincent Cohen-Addad}{equal,google}
\icmlauthor{Tommaso d'Orsi}{equal,google,bocconi}
\icmlauthor{Silvio Lattanzi}{equal,google}
\icmlauthor{Rajai Nasser}{equal,google}
\end{icmlauthorlist}

\icmlaffiliation{google}{Google Research}
\icmlaffiliation{bocconi}{BIDSA, Bocconi}

\icmlkeywords{Stochastic block models, Machine Learning, clustering, late fusion, early fusion, heterogenoeus, ICML}

\vskip 0.3in


\printAffiliationsAndNotice{\icmlEqualContribution} 

\begin{abstract}
    Graph clustering is a central topic in unsupervised learning with a multitude of practical applications. In recent years, multi-view graph clustering has gained a lot of attention for its applicability to real-world instances where one has access to multiple data sources. In this paper we formalize a new family of models, called \textit{multi-view stochastic block models} that captures this setting. 
    For this model, we first study efficient algorithms that naively work on the union of multiple graphs. Then, we introduce a new efficient algorithm that provably outperforms previous approaches by analyzing the structure of each graph separately. 
    Furthermore, we complement our results with an information-theoretic lower bound studying the limits of what can be done in this model.
    Finally, we corroborate our results with experimental evaluations.
\end{abstract}

\section{Introduction}\label{section:introduction}
Clustering graphs is a fundamental topic in unsupervised learning. It is used in a variety of fields, including data mining, social sciences, statistics, and more. The goal of graph clustering is to partition the vertices of the graph into disjoint sets so that similar vertices are grouped together and dissimilar vertices lie in different clusters. In this context, several notions of similarity between vertices have been studied throughout the years resulting in different clustering objectives and clustering algorithms~\cite{von2007tutorial, ng2001spectral, bansal2004correlation, goldberg1984finding, dasgupta2016cost}. 

Despite the rich literature,  most of the algorithmic results in graph clustering only focus on the setting where a single graph is presented in input. This is in contrast with the increasing practical importance of multimodality and with the growing attention in applied fields to multi-view or multi-layer clustering~\cite{paul2016consistent,corneli2016exact,pmlr-v37-hanb15,PhysRevE.95.042317,khan2019approximate, zhong2021latent, hu2019deep, abavisani2018deep, kim2016joint, gujral2020beyond, ni2016self, de2023mixture, papalexakis2013more, gujral2018smacd, gorovits2018larc}. 
In practice it is in fact observed that while a single data source only offers a specific characterization of the underlying objects, leading to a coarse partition of the data, a careful combination of multiple views often allows a semantically richer network structure to emerge \cite{fu2020overview, fang2023comprehensive}. 
For a practical example, consider the task of clustering users of a social network platform like Facebook, Instagram or X. To solve such task one could simply cluster the friendship graph, or one could cluster such graph by looking together at the friendship graph, the co-like graph (a graph where two users are connected if they like the same picture/video), the co-comment graph (a graph where two users are connected if they comment on the same post), the co-repost graph (a graph where two users are connected if they repost the same post) and so on and so forth. In practice, one would expect the second approach to work better in many settings because it provides a more fine-grained description of the behaviors of the users.

Despite the large number of basic applications, very little is known on the theoretical aspect
of the problem. Several works~\cite{paul2016consistent,corneli2016exact,pmlr-v37-hanb15,PhysRevE.95.042317} consider the multi-layer stochastic block models where the goal is to identify
$k$ communities given several instances (i.e., layer or view) of the stochastic block model, 
each with $k$ communities. In this paper, we would
like to work in a more general and more realistic setup, where there are $k$ communities but each instance only provides \emph{partial information} about these $k$ communities.
Very recently and concurrently to us,~\cite{de2023mixture} introduced the \textit{multi-view stochastic block model}. In this model, one is given in input multiple graphs, each coming from a stochastic block model, and the goal is to leverage the information contained in the graphs to recover the underlying clustering structure.
More precisely, given a vector of labels\footnote{We write random variables in boldface.} $\mathbf z$  where the labels capture the clustering assignment and are in $[k]$, and $t$ graphs $\mathbf G_1,\dots,\mathbf G_t$, where each graph $\mathbf G_\ell$ is drawn independently from a stochastic block model with $2$ labels and possibly distinct parameters, we are interested in designing an algorithm to weakly recover the underlying vector $\mathbf z$. One important aspect of the model is that none of the input graphs $\mathbf G_1,\dots,\mathbf G_t$ may contain enough information to recover the full clustering structure (for example because $2 < k$), nevertheless one can show that if enough graphs are observed it is possible to recover the clustering structure of the underlying instance. 

Armed with this new model we study different approaches to cluster the input graphs $\mathbf G_1,\dots,\mathbf G_t$. First, we note that the natural approach (sometimes used in practice) of merging the graphs and then clustering the union of the graphs, called \emph{early fusion}, leads to suboptimal results. Then we design a  more careful \emph{late fusion} clustering algorithm that first clusters all the graphs separately and then carefully merges their results. This shows the superiority of late over early fusion. Finally, we complement our results with an information-theoretic lower bound studying the limits of what can be done in this model. The bounds obtained are a drastic improvement over the ones obtained by~\cite{de2023mixture}.

\paragraph{Model}

Before formally introducing our model, we recall the classic definition of the stochastic block model. 

The  $k$ community symmetric stochastic block model (see \cite{abbe2017community} for a survey) denotes the following joint distribution $(\mathbf{x}, \mathbf{G})\sim \sbmk$ over a vector of $n$ labels in $[k]$ and a graph on $n$ vertices:
\begin{itemize}
    \item draw $\mathbf{x}$ from $[k]^n$ uniformly at random;
    \item for each distinct $i,j \in [n]$, independently create an edge $\ij$ in $\mathbf{G}$ with probability $(1+(1-\frac{1}{k})\eps)\tfrac{d}{n}$ if $\mathbf{x}_i = \mathbf{x}_j$ and probability $(1-\frac{\eps}{k})\frac{d}{n}$ otherwise.
\end{itemize}
We denote the conditional distribution of $\mathbf{G}$ given $\mathbf{x}=x$ as $\sbmkx{x}\,.$ Given a graph $\mathbf G$ sampled according to this model, the goal is to recover the (unknown)
underlying vector of labels as well as possible. 

Most of the statistical and computational phenomena at play can already be observed  in the simplest settings with two communities, so we will often focus on those. For $k=2$, we denote the distribution by $\sbm\,,$ i.e., we explicitly replace the subscript $k$. It will also be convenient to use $\{+1,-1\}$ for the community labels instead of $[2]$, so we will sometimes do this. The labeling convention should be clear from the context.

One of the most widely studied natural objective in the context of stochastic block models is that of \textit{weak recovery} --asking to approximately recover the communities.
Specifically, we say that an algorithm achieves weak recovery for $\set{\sbmk}_{n\in \N}$ if the correlation of the algorithm’s output $\hat{\mathbf{x}}(\mathbf{G})\in [k]^n$ and the underlying vector $\mathbf{x}$ of labels is better than random as $n$ grows,\footnote{We use $o(1)$ to denote a function $f$ such that $\lim_{n\rightarrow \infty} f(n)/1=0\,.$}
\begin{equation}
\label{eq:k-weak-recovery}
    \bbP\Paren{R(\hat{\mathbf x}(\mathbf{G}_\ell),\mathbf{x}_\ell)\geq \frac{1}{k}+\Omega_{d,\eps/k}(1)}\geq 1-o(1)\,,
\end{equation}
where $R(\hat{x},x)$ is the agreement between $\hat{x}$ and $x$, defined as\footnote{We use Iverson's brackets to denote the indicator function.}
\begin{equation}\label{eq:agreement}
    R(\hat{x},x)=\max_{\pi\in P_k} \frac{1}{n}\sum_{i\in[n]}\bracbb{\hat{x}_i=\pi(x_i)}\,,
\end{equation}
and $P_k$ is the permutation group of $[k]\,.$ 

A sequence of works \cite{decelle2011asymptotic, massoulie2014community,mossel2014belief, mossel2015reconstruction, abbe2016achieving, mossel2018proof, montanari2016semidefinite}, have studied the statistical and computational landscapes of this objective, with great success. 
The emerging picture \cite{bordenave2015non, abbe2016achieving,  montanari2016semidefinite} shows that it is possible to achieve weak recovery in polynomial time whenever $d\eps^2/k^2>1$, this value is called the Kesten-Stigum threshold.  Further evidence \cite{hopkins2017efficient} suggests that this threshold is optimal for polynomial time algorithms. In particular, for the special case of weak recovery with $2$ communities \cite{mossel2015reconstruction}  showed that the problem is solvable (also computationally efficiently) \textit{if and only if} $d\eps^2/4>1$. For larger values of $k$ a gap between information-theoretic results and efficient algorithms exists \cite{abbe2016crossing, banks2016information}.

In the context of multimodality, we define the following multi-view model. 

\begin{model}[Multi-View stochastic block model]\label{model:heterogeneous-sbm}
Let $k\geq 1$ and let $\cT$ be a sequence of $t$ tuples $(d_\ell,\eps_\ell)$ where $ d_\ell\geq 0$ and $ \eps_\ell\in (0,1)\,.$ 
We refer to the following joint distribution $(\mathbf{z},(\mathbf{f_1,\mathbf{G}_1}),\ldots(\mathbf{f}_t,\mathbf{G}_t))\sim\hsbm$ as the $(\cT,k,t)$-multi-view stochastic block model:
\begin{enumerate}
    \item for each $\ell\in [t]$, independently draw a mapping $\mathbf{f}_\ell:[k]\rightarrow \set{\pm 1}$ uniformly at random;
    \item independently draw a vector $\mathbf{z}$ from $[k]^n$ uniformly at random;
    \item for each $\ell\in [t]$, independently draw a graph $\mathbf{G}_\ell\sim \SBM_{n, 2, d_\ell,\eps_\ell}(\mathbf{f}_\ell(\mathbf{z}))$, where $\mathbf{f}_\ell(\mathbf{z})$ is the $n$-dimensional vector with entries $\mathbf{f}_\ell(\mathbf{z}_1),\ldots,\mathbf{f}_\ell(\mathbf{z}_n)\,.$
\end{enumerate}
Given $\mathbf{G}_1,\ldots,\mathbf{G}_t$, the goal is to approximately recover the unknown vector $\mathbf{z}$ of labels.

When $\cT=\set{(d_\ell,\eps_\ell)}_{\ell\in [t]}$ is such that $(d_\ell, \eps_\ell)=(d,\eps)$ for some $d,\eps,$ we denote the model simply by $\hsbmt$. 
\end{model}

Although \cref{model:heterogeneous-sbm} captures the algorithmic phenomena of multi-view models used in practice, more general versions of \cref{model:heterogeneous-sbm} could be defined, we discuss them in \cref{section:conclusions}.
Similarly to the vanilla stochastic block model, weak recovery can also be defined for \cref{model:heterogeneous-sbm}. %
We say that an algorithm achieves weak recovery for $\hsbm$ with $t$ observations, if it outputs a vector $\hat{\mathbf{z}}(\mathbf{G}_1,\ldots\mathbf{G}_t)$ satisfying:
\begin{align}\label{eq:heterogeneous-weak-recovery}
    \bbP\Paren{R(\hat{\mathbf{z}}(\mathbf{G}_1,\ldots\mathbf{G}_t),\mathbf{z})\geq \frac{1}{k}+\Omega(1)}\geq 1-o_t(1)\,.
\end{align}

Differently from the vanilla stochastic block model, the complexity of \cref{model:heterogeneous-sbm} is governed both by the $\SBM$ parameters in $\cT$ \textit{and} by the number of observations $t$. A good algorithm should then achieve weak recovery with the best possible multiway tradeoff between the edge-densities of the graphs, the biases and the number of observations at hand.
That is, extract as much information as possible so to require as few observations as possible. This novel interplay of parameters immediately raises two natural questions, which are the main focus of this work.

\textit{How many observations are needed?} 
The problem gets easier the larger the number of observations one has access to (see \cref{section:union-graph} for a formal proof).
It is also easy to see that for $t = o(\log k)$, it is information theoretically \textit{impossible} to approximately recover the communities (since $\log_2 k$ bits are needed to encode $k$ labels). Furthermore, as we will see, stronger lower bounds can also be obtained.

\textit{How many observations suffice?}
To understand how many observations suffice to recover the communities, it is instructive to consider the union graph $\mathbf{G}^*=\bigcup_{\ell\in [t]}\mathbf{G}_\ell$ of an instance from $\hsbmt$, which turns out to follow a $k$-communities stochastic block model with parameters $d^*=\Theta(dt)\,,\eps^*=\Theta(\eps)$ (see \cref{section:union-graph}). 
Building on the aforementioned results, this implies that at least $t\ge \Omega(k^2/d\eps^2)$ observations are needed for efficient weak recovery of the communities from $\mathbf{G}^*$! However, as we show later, exponentially better algorithms can bridge this gap.

\subsection{Results}\label{section:results}

\paragraph{Weak recovery}
Our main algorithmic result shows that weak recovery for $\hsbm$ can be achieved in polynomial time with only $O(\log k)$ many observations.

\begin{theorem}[Weak recovery for multi-view models]\label{theorem:main-algorithm}
Let $n,k>0\,.$ 
Let $(\mathbf{z},(\mathbf{f_1,\mathbf{G}_1}),\ldots(\mathbf{f_t},\mathbf{G}_t))\sim\hsbm$ for a sequence of tuples $\cT=\set{(d_\ell,\eps_\ell)}_{\ell=1}^t\,,$ each satisfying $d_\ell\cdot \epsilon_\ell^2/4> 1$.
Then there exists a constant $C_{\cT}>0$ depending only on $\cT$, such that if $t\geq \Omega \Paren{\frac{\log k}{C^2_{\cT}}}$,  weak recovery of $\mathbf{z}$ in the sense of \eqref{eq:heterogeneous-weak-recovery} is possible. Moreover, the underlying algorithm runs in polynomial time.
\end{theorem}

\cref{theorem:main-algorithm} implies that whenever the algorithm has access to $\Theta(\log k)$ observations, each above the relative Kesten-Stigum threshold, the guarantees of the underlying algorithm match the aforementioned trivial lower bound, up to constant factors. 
Moreover, as we will see in \cref{section:algorithm}, the underlying algorithm turns out to be surprisingly simple and efficient.

The algorithm in \cref{theorem:main-algorithm} applies a specialized weak-recovery algorithm on each view $\mathbf{G}_\ell$ to obtain a matrix $\hat{\mathbf{X}}_{\ell}$ estimating $\dyad{\mathbf{f}_\ell(\mathbf{z})}$ and achieving the correlation
\begin{equation}
\label{equation:definition-specialized-weak-recovery-pairwise}
\resizebox{0.47\textwidth}{!}{$
    C_{\ell}\leq \E\Brac{\hat{\mathbf X}(\mathbf{G})_{\ij}\suchthat \mathbf{f}_\ell(\mathbf{z})_i=\mathbf{f}_\ell(\mathbf{z})_j} 
     - \E\Brac{\hat{\mathbf X}(\mathbf{G})_{\ij}\suchthat \mathbf{f}_\ell(\mathbf{z})\neq \mathbf{f}_\ell(\mathbf{z})_j}
    \,,$}
\end{equation}
where $C_\ell$ depends only on $d_\ell,\eps_\ell$. The algorithm then proceeds into processing the outputs $\hat{\mathbf{X}}_{\ell}$ \emph{in a blackbox fashion} to produce an estimate $\hat{\mathbf{z}}$ of $\mathbf{z}$.

The constant $C_\mathcal{T}$ in \cref{theorem:main-algorithm} is the average of the correlations $(C_\ell)_{\ell\in[t]}$. It is natural to wonder whether the dependency of the number of observations on $C_\mathcal{T}$ is needed. Moreover, as for canonical stochastic block models,  it is natural to ask what the exact phase transition of \cref{model:heterogeneous-sbm} is. 
While we leave this latter fascinating question open, our next result shows that if we want an algorithm that processes estimates 
in a blackbox fashion, then some dependency on $C_\cT$ is indeed needed.


\begin{theorem}[Lower bound for multi-view models - Informal]\label{theorem:lowerbound}
Let $n,k>0\,.$ Let $(\mathbf{z},(\mathbf{f_1,\mathbf{G}_1}),\ldots(\mathbf{f_t},\mathbf{G}_t))\sim\hsbm$ for a sequence of tuples $\cT=\set{(d_\ell,\eps_\ell)}_{\ell=1}^t\,,$ each satisfying $d_\ell\cdot \epsilon_\ell^2/4 > 1$. Assume that for every $\ell\in[t]$ we have an estimate\footnote{Such estimates might be obtained by applying a blackbox weak-recovery algorithm for $\sbm$ on each of $\mathbf{G}_1,\ldots,\mathbf{G}_t$, and which has the mentioned correlation guarantee.} $\hat{\mathbf{X}}_\ell$ of $\dyad{\mathbf{f}_\ell(\mathbf{z})}$ satisfying a pair-wise correlation (as in \eqref{equation:definition-specialized-weak-recovery-pairwise}) of at least $C_\ell>0\,,$ and let $C_\cT$ be the average correlation.

If $t= o \Paren{\frac{\log k}{C_\cT}}$, then by only using the estimates $\hat{\mathbf{X}}_1,\ldots,\hat{\mathbf{X}}_t$, it is information-theoretically impossible to return a vector $\hat{\mathbf{z}}$ satisfying
\begin{equation}
\label{equation:z-weak-recovery-in-lower-bound-theorem}
\bbP\Paren{
R(\hat{\mathbf{z}},\mathbf{z})\geq \frac{1}{k}+\Omega(1)}\geq 0.99 \,.
\end{equation}
\end{theorem}



\paragraph{Exact recovery}
Another widely studied objective for stochastic block models is that of \textit{exact recovery}, where the goal is to correctly classify all vertices in the graph (see \cite{abbe2015exact,mossel2015consistency, abbe2017community} and references therein).
In the context of \cref{model:heterogeneous-sbm} this objective becomes
\begin{align}\label{eq:exact-recovery-hsbm}
    \bbP \Paren{R(\hat{\mathbf{z}},\mathbf{z} )=1}\geq 1-o(1)\,.
\end{align}
As a corollary we show that, when given access to more views, the algorithm behind \cref{theorem:main-algorithm} can achieve exact recovery.
\begin{corollary}[Exact recovery for multi-view models]\label{corollary:exact-recovery-main}
Consider the settings of \cref{theorem:main-algorithm}, if $t\geq \Omega\Paren{\frac{\log n}{C_\cT^2}}$ then exact recovery of $\mathbf{z}$ in the sense of \cref{eq:exact-recovery-hsbm} is possible. Moreover, the underlying algorithm runs in polynomial time.
\end{corollary}

\paragraph{Experiments}
\cref{theorem:main-algorithm} show hows, for \cref{model:heterogeneous-sbm}, \textit{late fusion} algorithms can provide better guarantees --by requiring an exponentially smaller number of observations to achieve the same error in a large parameters regime-- than early fusion algorithms.
We further corroborated these findings with experiments on synthetic data in \cref{section:experiments}.

\section*{Organization}
The rest of the paper is organized as follows.
We introduce the main ideas in \cref{section:techniques}. 
In \cref{section:specialized-weak-recovery} we introduce our specialized weak recovery algorithm for the standard stochastic block model. This is then used in the design of the algorithm behind \cref{theorem:main-algorithm} in \cref{section:algorithm}.
Experiments are presented in \cref{section:experiments}.
Future directions and conclusions are discussed in \cref{section:conclusions}.
In \cref{section:union-graph} we show the limits of algorithms using the union graph. \cref{section:lowerbound} contains a proof of (the formal version of) \cref{theorem:lowerbound}.
Deferred proofs are presented in \cref{section:deferred-proofs}.

\section*{Notation}
We denote random variables in \textbf{boldface}. We hide constant factors using the notation $O(\cdot), \Omega(\cdot), \Theta(\cdot)\,.$ We write $O_\delta(\cdot)\,,\Omega_\delta(\cdot)$ to specify that the hidden constant may depend on the parameter $\delta$. Similarly, we sometimes write $C_\delta$ to denote a constant depending only on $\delta$. We further denote the indicator function with Iverson's brackets $\bracbb{\cdot}\,.$ Given functions $f,g:\R\rightarrow\R$, we say $f\in o_n(g)$ if $\lim_{n\rightarrow\infty}f(n)/g(n)=0$. Similarly we write $f\in\omega_n(g)$ if $g\in o_n(f)\,.$ 
With a slight abuse of notation we often write $o_n(g)$ to denote a function in $o_n(g).$ When the context is clear we drop the subscript. In particular, we often
write $o(1)$ to denote functions that tends to zero as $n$ grows. We say that an event happens with high probability if this probability is at least $1-o(1)\,.$ For a set $S\subseteq [n]$, we write $\mathbf i\overset{u.a.r.}{\sim}S$ to denote an element drawn uniformly at random.
For a given probability distribution and a measurable event $\cE\,,$ we denote the probability that the event occurs by $\bbP(\cE)\,.$ We denote the complement event by $\cE^c\,.$

For a vector $v\in \R^n$, we write $\Norm{v}$ for its Euclidean norm. 
For a matrix $M\in\R^{n \times n}$, we denote by $\Norm{M}$ its spectral norm and by $\Normf{M}$ its Frobenius norm. We also let $\Normo{M}:=\sum_\ij \Abs{M_\ij}\,.$ We denote the $i$-th row of $M$ by $M_i$.
For a graph $G$ with $n$ vertices, we denote by $A(G)$ its adjacency matrix. When the context is clear we simply write $A$. If the graph is directed, row $A_i$ contains the outgoing edges of vertex $i\,.$

For a given vector of labels $z\in [k]^n$, we denote by $c_1(z)\ldots,c_k(z)$ the $n$-dimensional indicator vectors of the $k$ communities defined by $z$.  In the interest of simplicity, we often denote instances $(\mathbf{z},(\mathbf{f_1,\mathbf{G}_1}),\ldots(\mathbf{f_t},\mathbf{G}_t))$ drawn from $\hsbm$ simply by $\bm \cI\,.$
For $z\in[k]^n$, we also denote by $\hsbm(z)$ the distribution of $(\mathbf{z},(\mathbf{f_1,\mathbf{G}_1}),\ldots(\mathbf{f_t},\mathbf{G}_t))\sim \hsbm$ conditioned on the event $\mathbf{z}=z\,.$ We often call $z\in [k]^n$ a ``community vector".

We say that an algorithm runs in time $T$, if in the worst case it performs at most $T$ elementary operations.

\section{Techniques}\label{section:techniques}
We outline here the main ideas behind \cref{theorem:main-algorithm} and \cref{theorem:lowerbound}. In the interest of clarity, we limit our discussion to $\hsbmt$.

\paragraph{Behavior of the union graph}
The algorithm behind \cref{theorem:main-algorithm} is remarkably simple and leverages  known algorithms for weak recovery of stochastic block models (particularly related to the robust algorithms of \cite{ding2022robust, ding2023node}).
As a first step, to gain intuition,  it is instructive to understand why the union graph instead requires $t\geq \Omega(k^2)$ observations (see \cref{theorem:limits-weak-recovery-union}). 
An instance $\bm \cI$ of $\hsbmt$ is given by a vector $\mathbf{z}\in [k]^n$ and a collection of $t$ independent pairs $(\mathbf{f}_1,\mathbf{G}_1)\,,\ldots\,, (\mathbf{f}_t,\mathbf{G}_t)$, where each $\mathbf{G}_\ell$ is sampled from $\sbmx{\mathbf{f}_\ell(\mathbf z)}\,.$ Since, by definition, each edge $\{i,j\}$ appears in $\mathbf{G}_\ell$ with probability $\displaystyle\Paren{1+\frac{\epsilon}{2}\cdot\mathbf{f}_\ell(\mathbf{z})_i\cdot \mathbf{f}_\ell(\mathbf{z})_j}\cdot\tfrac{d}{n}\,,$ in the union graph $\bigcup_{\ell\in [t]}\mathbf{G}_\ell$ the same edge appears roughly with probability
\begin{align*}
    \frac{d}{n}\cdot \sum_{\ell\in [t]}\Paren{1+\frac{\epsilon}{2}\cdot\mathbf{f}_\ell(\mathbf{z})_i\cdot \mathbf{f}_\ell(\mathbf{z})_j}\,.
\end{align*}
The intuition here is that if $\mathbf{z}_i\neq \mathbf{z}_j$ the second term in the sum would be close to $0$, while for $\mathbf{z}_i= \mathbf{z}_j$ it would be $\frac{\eps t}2\,.$
Hence, we will see this edge in the union graph roughly with probability 
\begin{align*}
    &\frac{d t}{n} \cdot (1+O(\eps)\cdot \bracbb{\mathbf{z}_i= \mathbf{z}_j})\,.
\end{align*}
That is, the union graph behaves similarly to a vanilla stochastic block model with $k$ communities, bias $O(\eps)$ and expected degree $d t\,.$
As stated in the introduction, existing efficient algorithms can achieve weak recovery for that distribution whenever $(d t) \eps^2/k^2>\Omega(1)\,,$ implying $t\geq \Omega(k^2)$ in the regime $4<d\eps^2\leq O(1)$ where weak recovery for each observation is possible.

\paragraph{Amplifying the signal-to-noise ratio via black-box estimators}
The above approach of taking the union graph and then running community detection on it yields sub-optimal guarantees because the graph does not keep all information regarding the instance $\bm \cI\,.$
Our strategy to overcome this issue is to proceed in the reverse order: \textit{first} extract as much information as possible from each graph, and \textit{then} combine the data.
Concretely, in the context of $\hsbmt$, our plan is to accurately estimate the matrix $\mathbf{f}_\ell(\mathbf{z})\transpose{\mathbf{f}_\ell(\mathbf{z})}$ for each graph $\mathbf{G}_\ell\,.$
Indeed, the polynomial $\sum_{\ell \in [t]}\mathbf{f}_\ell(\mathbf{z})_i\mathbf{f}_\ell(\mathbf{z})_j$ strongly correlates with $\bracbb{\mathbf{z}_i=\mathbf{z}_j}$ in the sense that
\begin{align*}
    t=&\E\Brac{\sum_{\ell \in [t]}\mathbf{f}_\ell(\mathbf z)_i\mathbf{f}_\ell(\mathbf z)_j\given \mathbf z_i=\mathbf z_j}
    >\E\Brac{\sum_{\ell \in [t]}\mathbf{f}_\ell(\mathbf z)_i\mathbf{f}_\ell(\mathbf z)_j\given \mathbf z_i\neq \mathbf z_j}=0\,.
\end{align*}
In other words, we can accurately estimate whether $\mathbf{z}_i=\mathbf{z}_j$ or not by accurately estimating the products $\sum_{\ell \in [t]}\mathbf{f}_\ell(\mathbf z)_i\mathbf{f}_\ell(\mathbf z)_j\,.$
Now, we do not have access to the functions $\mathbf{f}_\ell(\mathbf{z})$ but we can hope (see the subsequent paragraphs) that existing weak recovery algorithms can provide a close enough estimate in the sense 
\begin{align}
    t\cdot C_{d,\eps}&=\E\Brac{\sum_{\ell \in [t]}\hat{\mathbf{x}}(\mathbf{G}_\ell)_i\hat{\mathbf{x}}(\mathbf{G}_\ell)_j\given \mathbf z_i=\mathbf z_j}\label{eq:techniques-separation}>\E\Brac{\sum_{\ell \in [t]}\hat{\mathbf{x}}(\mathbf{G}_\ell)_i\hat{\mathbf{x}}(\mathbf{G}_\ell)_j\given \mathbf z_i\neq \mathbf z_j}=0\,.
\end{align}
If so, we may simply decide whether $i,j$ should be clustered together based on how large $\sum_{\ell \in [t]}\hat{\mathbf{x}}(\mathbf{G}_\ell)_i\hat{\mathbf{x}}(\mathbf{G}_\ell)_j$ is.
By independence of the observations, standard concentration of measure results tell us that
$\Omega(\log (n)/C_{d,\eps}^2)$ observations\footnote{This is better than $\Omega(k^2)$ as long as $k\geq \Omega(\sqrt{\log n})$.} would suffice to \textit{exactly} predict all the $n^2$ pairs (and hence achieve exact recovery with this number of observations).

\paragraph{Improvements via neighborhoods intersection}
Continuing with the above line of thinking, one can further improve the dependency on $t$ to $t=\Theta(\log k)$ as promised in \cref{theorem:main-algorithm}.
For a label $p\in [k]\,,$ let $c_p(z)\in \Set{0,1}^n$ be the indicator vector of the corresponding community.
The improvement comes from observing that for $\ell\neq \ell'$ and for any \textit{typical} labelling $z$ (i.e. a labelling that is approximately balanced), we have large separation between the community indicator vectors $\Snorm{c_p(z)-c_{p'}(z)}\geq \Omega(n/k)\,.$
The crucial consequence is that for a fixed index $i\in[n]$, we do not need to guess correctly $\bracbb{z_i=z_j}$ for all $j$ and  we may misclassify some pairs. Indeed if ${A}_i\in \Set{0,1}^n$ is a vector with entries $({A}_i)_j$ that accurately predicts $\bracbb{z_i=z_j}$ up to a $\rho < n/k$ misclassification error, then we can deduce whether $i,j$ come from the same community by verifying if ${A}_i$ and ${A}_j$ agree on the majority of their entries. Concretely, by the reverse triangle inequality it holds that
\begin{align*}
    &\Abs{\Norm{{A}_i-{A}_j}-\Norm{c_p(z)-c_{p'}(z)}} \leq \Norm{c_p(z)-{A}_i} + \Norm{c_{p'}(z)-{A}_j}\leq O(n/k)\,.
\end{align*}
That is, we are still able to exactly deduce whether $i,j$ sit in the same community!
The improvement over $t$ then comes as $O(\log(k)/C_\delta^2)$ observations suffice to bound the misclassification error by $n/k$ times a tiny constant.
Finally, we remark that we are bound to misclassify some vertices as for some vertices $i$ the estimator vector $\mathbf{A}_i$ will not accurately represent its community when $t\leq O(\log(k)/C_\delta^2)$ (this is due to the well-known gap between weak recovery and exact recovery in the vanilla stochastic block model \cite{abbe2017community}).

\paragraph{The pair-wise weak recovery estimator}
So far, we glossed over the fact that we do not have an algorithm returning an accurate estimate $\hat{\mathbf{x}}(\mathbf{G}_\ell)\transpose{\hat{\mathbf{x}}(\mathbf{G}_\ell)}$ of $\mathbf{f}_\ell(\mathbf{z})\transpose{\mathbf{f}_\ell(\mathbf{z})}$ given the graph $\mathbf{G}_\ell\,.$
Notice that for a pair $(\mathbf{x},\mathbf{G})\sim \sbm$, an algorithm achieving weak recovery returns a vector $\hat{\mathbf{x}}(\mathbf{G})\in \Set{\pm 1}^n$ such that
\begin{align*}
    \Omega(n^2)&\leq \E \Brac{\iprod{\hat{\mathbf{x}}(\mathbf{G}), \mathbf{x}}^2}\leq \E \Brac{\iprod{\dyad{\hat{\mathbf{x}}(\mathbf{G})}, \dyad{\mathbf{x}}}}
\end{align*}
which is enough to obtain the separation required in \cref{eq:techniques-separation}. Indeed, the above implies that on average
\begin{align*}
    \E \Brac{\hat{\mathbf{x}}(\mathbf{G})_{\mathbf i}\hat{\mathbf{x}}(\mathbf{G})_{\mathbf j} \given \mathbf{x}_{\mathbf i} = \mathbf{x}_{\mathbf j}} - \E \Brac{\hat{\mathbf{x}}(\mathbf{G})_{\mathbf i}\hat{\mathbf{x}}(\mathbf{G})_{\mathbf j} \given \mathbf{x}_{\mathbf i} \neq \mathbf{x}_{\mathbf j}}
    \geq \Omega_{d,\eps}(1)\,,
\end{align*}
which is enough to carry out the strategy outlined in the previous paragraphs.
Notice that, a priori, it is not clear whether these estimators should work for $\hsbmt$ as the model introduces some subtle difficulties compared to $\sbm\,.$ Most importantly, the labels  in $\mathbf{f}_\ell(\mathbf{z})$ --and hence the edges in $\mathbf{G}_\ell$-- are \textit{not} pair-wise independent.

We bypass this obstacle carrying out the analysis \textit{after} conditioning on the choice of $\mathbf{f}_\ell$, so that, even though the communities may be unbalanced, the edges are again independent.
Now, for highly unbalanced communities the expected degree of each vertex is an accurate predictor for its community, hence weak recovery is easy to achieve. On the other hand, one can treat slightly unbalanced communities as \textit{perturbed} balanced communities and apply the node-robust algorithm of \cite{ding2023node}.

\begin{remark}[Connection with \cite{liu2022statistical}]
Liu, Moitra and Raghavendra studied a joint distribution model over hypergraphs with independence edges.
In the special case of graphs, their model can be seen as a simpler version of \cref{model:heterogeneous-sbm} in which every $\mathbf{f}_\ell(\cdot)$ is \textit{known}, and each $d_\ell$ is a constant.
For this model, \cite{liu2022statistical} beats random guessing: They produce a unit vector that correlates with the community vector better than a random vector guess would. However, they provide no rounding strategy.
The algorithmic techniques in \cite{liu2022statistical} are very different from ours and do not imply a result of the form of \cref{theorem:main-algorithm} for \cref{model:heterogeneous-sbm}.
\end{remark}

\paragraph{Information theoretic lower bounds for black-box algorithms}

The proof of \cref{theorem:lowerbound} consists of two main steps. In the first step, we bound how much information an estimate $\hat{\mathbf{X}}_\ell$ with pair-wise correlation $C_\ell$ can reveal about $\mathbf{z}$. We show that this can be bounded (in terms of mutual information) as $I(\mathbf{z};\hat{\mathbf{X}}_\ell)\leq O(C_\ell\cdot n)$. In the second step we use an adapted version of Fano's inequality to show that if $\hat{\mathbf{z}}$ is an estimate of $\mathbf{z}$ achieving \eqref{equation:z-weak-recovery-in-lower-bound-theorem}, then $\hat{\mathbf{z}}$ must have at least $\Omega(n\cdot\log k)$ bits of information about $\mathbf{z}$, i.e., $I(\mathbf{z};\hat{\mathbf{z}})\geq \Omega(n\cdot\log k)$.

Now if $\hat{\mathbf{z}}$ is obtained by only processing $(\hat{\mathbf{X}}_\ell)_{\ell\in[t]}\,,$ then by using the data processing inequality, we can show that
\begin{align*}
    \Omega(n\cdot\log k)&\leq \sum_{t\in[t]}I(\mathbf{z};\hat{\mathbf{X}}_\ell)\leq \sum_{t\in[t]}O(C_\ell\cdot n)\leq O(C_\cT\cdot t\cdot n)\,,
\end{align*}
where $C_\cT$ is the average of the correlations $(C_\ell)_{\ell\in[t]}\,.$ From this we deduce that we must have $t\geq\Omega\Paren{\frac{\log k}{C_\cT}}\,.$

\section{Specialized weak recovery for vanilla SBMs}\label{section:specialized-weak-recovery}


General weak recovery results \cite{abbe2017community} for stochastic block models turn out to be too weak for our objective. 
We rely instead on the following stronger statement, which we obtain by exploiting the robust algorithms of \cite{ding2022robust}, \cite{ding2023node}, and which also works down to the Kesten-Stigum threshold. 
This specialized estimator will be used in the main algorithm behind \cref{theorem:main-algorithm}.

\begin{theorem}[Pair-wise weak recovery for unbalanced $2$ communities stochastic block model]\label{theorem:weak-recovery-sbm}
Let $n,d,\eps>0$ be satisfying $d \eps^2/4 -1>0\,.$ 
Let $\mathbf{x}=(\mathbf{x}_i)_{i\in [n]}\in \set{\pm 1}^n$ be a sequence of i.i.d. binary random variables with $\bbP(\mathbf{x}_i=+1)=p\,.$
There exists a polynomial time algorithm such that, on input $\mathbf{G}\sim\sbm(\mathbf{x})$, returns with probability $1-o(1)$ a matrix $\hat{\mathbf X}(\mathbf{G})\in [-1\,, +1]^{n\times n}$ satisfying $\forall i,j \in[n]$
\begin{align*}
    C_{d,\eps}\leq \E\Brac{\hat{\mathbf X}(\mathbf{G})_{\ij}\suchthat \mathbf{x}_i=\mathbf{x}_j} - \E\Brac{\hat{\mathbf X}(\mathbf{G})_{\ij}\suchthat \mathbf{x}_i\neq \mathbf{x}_j}
\end{align*}
for some constant $C_{d,\eps}>0\,.$ 
\end{theorem}
Notice that the algorithm only receives $\mathbf{G},d,\eps$ and hence it \textit{does not} know $p\,.$ Note also that in the above, if one of the events $\mathbf{x}_i=\mathbf{x}_j$ or $\mathbf{x}_i\neq\mathbf{x}_j$ has probability 0 (i.e., if $p=0$ or $p=1$), then we adopt the convention that the corresponding conditional expectation is 0.

We defer the proof of \cref{theorem:weak-recovery-sbm} to \cref{section:deferred-proofs}.

\section{The algorithm}\label{section:algorithm}
In this section we prove \cref{theorem:main-algorithm}.
We start by describing the underlying algorithm. Throughout the section, for a given $t\,,$ we assume  $\cT=\set{(d_\ell,\eps_\ell)}_{\ell=1}^t$ to be a sequence of $t$ tuples $(d_\ell,\epsilon_\ell)$ each satisfying $d_\ell\cdot \epsilon^2_\ell/4-1>0\,.$ 
For each $\ell\in[t]$, let $C_{\ell}$ be the weak-recovery constant that is achievable for $\SBM_{n,d_\ell,\epsilon_\ell}$ in the sense of \Cref{theorem:weak-recovery-sbm} (recall this constant depends only on $d_\ell,\eps_\ell$) and let
$$\Cbar=\frac1t\sum_{\ell\in[t]}C_{\ell}>0\,.$$
We also assume $k\leq n^{1-\Omega(1)}\,.$




\begin{algorithm}[ht]
   \caption{Community detection for multi-view stochastic block models}
   \label{algorithm:heterogeneous-models}
\begin{algorithmic}
   \STATE {\bfseries Input:}  $ k, G_1,\ldots,G_t\,.$
    \STATE {\bfseries Output:} Community vector $\hat{\mathbf{z}}$ in $[k]^n\,.$ 
    \STATE {\bfseries For} each $G_\ell$ with $\ell\in [t]$, run the community detection algorithm of \cref{theorem:weak-recovery-sbm}.
    \STATE Construct the directed graph $\mathbf{F}$ on the vertex set $[n]$ as follows.
    \FOR{$i=1$ {\bfseries to} $n$} 
    \STATE Add an outgoing edge to the $n/k$ vertices $j\in [n]$ with largest $\sum_{\ell \in [t]} \hat{\mathbf{X}}(G_\ell)_\ij\,.$
    \ENDFOR
    \STATE Run \cref{algorithm:second-moment-rounding} on the adjacency matrix $\mathbf{A}$ of $\mathbf{F}$ and return the resulting vector.
\end{algorithmic}
\end{algorithm}

\begin{remark}[Running time]
By \cref{theorem:weak-recovery-sbm}, the first step of the algorithm takes time $O\Paren{t\cdot n^{O(1)}}\,.$
Computing the values of $\sum_{\ell \in [t]} \hat{\mathbf{X}}(G_\ell)_\ij$ for all pairs takes time $O(t\cdot n^2)\,.$  Drawing the edges takes time $O(n^2\log n)$.
Hence the algorithm runs in time $O\Paren{(tn)^{O(1)} + T}$ where $T$ is the running time of \cref{algorithm:second-moment-rounding}.
\end{remark}

\paragraph{Graph structure on balanced instances}
\cref{algorithm:heterogeneous-models} will work on typical instances from $\hsbm$. In particular, it will work on instances that are approximately balanced, as described below.

\begin{definition}[Balanced vector]
Let $z\in[k]^n\,.$
We say that $z$ is balanced if for all $p \in [k]\,,$ it holds 
$$\Paren{1-n^{-\Omega(1)}}\frac{n}{k}\leq \Snorm{c_p(z)}\leq \Paren{1+n^{-\Omega(1)}}\frac{n}{k}\,.$$
If $z$ is balanced we also say that $\bm \cI\sim \hsbm(z)$ is balanced.
\end{definition}
It is immediate to see that a random instance $(\mathbf{z},(\mathbf{f_1,\mathbf{G}_1}),\ldots(\mathbf{f}_t,\mathbf{G}_t))\sim\hsbm$ is balanced with high probability.

\begin{fact}[Probability of balanced instance]\label{fact:balanced-instance}
Let $\bm \cI:=(\mathbf{z},(\mathbf{f_1,\mathbf{G}_1}),\ldots(\mathbf{f}_t,\mathbf{G}_t))\sim\hsbm,.$ Then, with probability $1-n^{-10}\,,$ $\bm \cI$ is balanced.
\end{fact}
The proof of \cref{fact:balanced-instance} is straightforward and we defer it to \cref{section:deferred-proofs}.

On balanced instances with sufficiently many observations, the adjacency matrix $\mathbf{A}$ of $\mathbf{F}$ becomes a good approximation of the true community matrix $\sum_{i\in [k]} \dyad{c_i(z)}$.

\begin{lemma}[Graphs structure from good instances]\label{lemma:structure-of-graph}
Let $n,k, t>0$. Let $\bm \cI:=(\mathbf z,(\mathbf f_1,\mathbf G_1),\ldots(\mathbf f_t,\mathbf G_t))$. Then \cref{algorithm:heterogeneous-models} constructs an adjacency matrix $\mathbf{A}$ such that $\forall p \in [k]\,, \forall i \in [n]$
\begin{align*}
    \E&\Brac{\Snorm{\mathbf{A}_i-c_p(\mathbf{z})}\given  \mathbf{z}_i=p}\leq O(n)\cdot \Paren{n^{-\Omega(1)} +e^{-\Omega(\bar{C}^2\cdot t)}}\,. 
\end{align*}
\end{lemma}

To prove \cref{lemma:structure-of-graph} we require an intermediate step.

\begin{fact}\label{fact:concentration-edges}
Consider the setting of \cref{lemma:structure-of-graph}. There exists a constant $C^*\in [-t,t]$ such that, for $i,j\in [n]$,
\begin{align*}
    &
    \resizebox{0.47\textwidth}{!}{$\displaystyle
    \bbP\Paren{\sum_{\ell \in [t]} \hat{\mathbf{X}}(\mathbf G_{\ell})_\ij < C^*- \frac{\Cbar\cdot t}{3}\suchthat  \mathbf{z}_i= \mathbf{z}_j}
    \leq  e^{-\Omega\Paren{\Cbar^2\cdot t}}\,,
    $}\\
    &
    \resizebox{0.47\textwidth}{!}{$\displaystyle
    \bbP\Paren{\sum_{\ell \in [t]} \hat{\mathbf{X}}(\mathbf G_{\ell})_\ij\geq C^* -\frac{\Cbar\cdot t}{3}\suchthat  \mathbf{z}_i\neq \mathbf{z}_j}\leq e^{-\Omega\Paren{\Cbar^2\cdot t}}\,.
    $}
\end{align*}
\end{fact}
We defer the proof of \cref{fact:concentration-edges} to \cref{section:deferred-proofs}.
We are now ready to prove \Cref{lemma:structure-of-graph}.

\begin{proof}[Proof of \cref{lemma:structure-of-graph}]
Fix $i \in [n]\,.$
We can limit our analysis after conditioning on the event  $\cE(\mathbf{z})$  that $\mathbf{z}$ is balanced.
Indeed, we have\begin{align*}
    \E & \Brac{\Snorm{\mathbf{A}_i-c_p(\mathbf{z})}\given  \cE(\mathbf{z})^c\,,\mathbf{z}_i=p}\cdot \bbP\Paren{\cE(\mathbf{z})}\leq O(n)\cdot \bbP\Paren{\cE(\mathbf{z})}\leq n^{-\Omega(1)}\,.
\end{align*}
Let $C^*$ be the constant of \cref{fact:concentration-edges}. Define
$$\mathbf{B}_{ij}=\sum_{\ell\in [t]} \hat{\mathbf{X}}(\mathbf{G}_\ell)_\ij\,,$$
and let $\mathbf{A}'_i=(\mathbf{A}'_{ij})_{j\in [n]}$ be the binary vector defined as
$$\mathbf{A}'_{ij} = \Bracbb{\mathbf{B}_{ij}\geq C^*-\frac{\bar{C}\cdot t}{3}}\,.$$

Notice that $\Norm{\mathbf{A}_i'}_1$ is the number of indices $j$ with $\mathbf{B}_{ij}\geq C^*-\frac{\bar{C}\cdot t}{3}\,.$ Now from the definition of the binary vector $\mathbf{A}_i\,,$ we know that $\mathbf{A}_{ij}=1$ if and only if $\mathbf{B}_{ij}$ is among the top $n/k$ values in $\Set{\mathbf{B}_{ij'}:j'\in[n]}\,.$ From this it is not hard to see that that $\Norm{\mathbf{A}_i-\mathbf{A}_i'}_1$ can be bounded from above by how much $\Norm{\mathbf{A}_i'}_1$ deviates from $n/k\,,$ that is
\begin{align*}
    \Norm{\mathbf{A}_i-\mathbf{A}_i'}_1&\leq \Abs{\Norm{\mathbf{A}_i'}_1-\frac{n}{k}}\,.
\end{align*}
Now assuming that $\cE(\mathbf{z})$ holds, we have 
$$\Norm{c_p(z)}_1=\Norm{c_p(z)}^2= \frac{n}{k}\pm n^{1-\Omega(1)}\,,$$
hence
\begin{align*}
    \Norm{\mathbf{A}_i-\mathbf{A}_i'}_1&\leq \Abs{\Norm{\mathbf{A}_i'}_1-\Norm{c_p(z)}_1}+ n^{1-\Omega(1)}\leq \Norm{\mathbf{A}_i'-c_p(z)}_1+ n^{1-\Omega(1)}\,.
\end{align*}
Since $c_p(\mathbf{z})$ and $\mathbf{A}_i$ are binary vectors, we have
\begin{align*}
    \Norm{\mathbf{A}_i-c_p(z)}^2&= \Norm{\mathbf{A}_i-c_p(z)}_1\leq \Norm{\mathbf{A}_i-\mathbf{A}_i'}_1+\Norm{\mathbf{A}_i'-c_p(z)}_1\,,
\end{align*}
and hence, given $\cE(\mathbf{z})\,,$
\begin{equation}
\label{equation:inequality-between-norms-for-balanced}
    \Norm{\mathbf{A}_i-c_p(z)}^2\leq 2\Norm{\mathbf{A}_i'-c_p(z)}_1+ n^{1-\Omega(1)}\,.
\end{equation}
Now notice that
\begin{align*}
\Abs{(\mathbf{A}'_i)_j-c_p(\mathbf{z})_j}
&=
\resizebox{0.44\textwidth}{!}{$
\bracbb{(\mathbf{A}'_i)_j=1}\bracbb{c_p(\mathbf{z})_j=0}+\bracbb{(\mathbf{A}'_i)_j=0}\bracbb{c_p(\mathbf{z})_j=1}
$}\\
&=
\bracbb{(\mathbf{A}'_i)_j=1}\bracbb{\mathbf{z}_j\neq p}+\bracbb{(\mathbf{A}'_i)_j=0}\bracbb{\mathbf{z}_j= p}\,.
\end{align*}
Using \cref{fact:concentration-edges} and leveraging the fact that $\bbP\Paren{\mathcal{E}(\mathbf{z})^c \given \mathbf{z}_i=p}\leq O\Paren{n^{-10}}$ we get
\begin{align*}
    \E&\Brac{\Abs{(\mathbf{A}'_i)_j-c_p(\mathbf{z})_j}\given\mathcal{E}(\mathbf{z}), \mathbf{z}_i=p}\leq e^{-\Omega\Paren{\Cbar^2\cdot t}}+O\Paren{n^{-10}}\,.
\end{align*}
Combining this with \eqref{equation:inequality-between-norms-for-balanced} we conclude that
\begin{align*}
    \E&\Brac{\Snorm{\mathbf{A}_i-c_p(\mathbf{z})}\given  \mathcal{E}(\mathbf{z}),\mathbf{z}_i=p}\leq O(n)\cdot \Paren{n^{-\Omega(1)} +e^{-\Omega(\bar{C}^2\cdot t)}}\,.
\end{align*}

\end{proof}

\paragraph{Rounding} Thanks to \cref{lemma:structure-of-graph}, an application of the following rounding scheme --sometimes called \textit{second moment rounding}, as one may see $\mathbf{A}$ as an estimate of $\sum_p \dyad{c_p(z)}$-- suffices to compute the true communities.
\begin{algorithm}[ht]
   \caption{Second moment rounding}
   \label{algorithm:second-moment-rounding}
\begin{algorithmic}
   \STATE {\bfseries Input:}  A matrix $A\in \set{0,1}^{n\times n}\,.$
    \STATE {\bfseries Output:} Community vector $\hat{\mathbf{z}}$ in $[n]^k\,.$  
    \STATE Let $S=\emptyset\,.$
    \FOR{$p=1$ {\bfseries to} $k$ (stop earlier if $S=[n]$)}
    \STATE Pick uniformly at random $\mathbf i\in [n]\setminus \mathbf S\,.$ Set $\mathbf S_p =\set{\mathbf i}\,.$
    \FOR{$j\in [n]\setminus  \mathbf S\,, j\neq  \mathbf i$}
    \STATE Set $j\in \mathbf S_p$ if $\Snorm{A_{ \mathbf i}-A_j}\leq \frac{n}{k}\,.$
    \ENDFOR
    \STATE Add $\mathbf{S}_p$ to $ \mathbf S\,.$
    \ENDFOR
    \STATE Assign each $i\in [n]\setminus \mathbf S$ to a set $\mathbf S_{\mathbf p}$, where $\mathbf{p}$ is chosen uniformly at random.
    \STATE {\bfseries Return:} the vector $\hat{\mathbf{z}}$ with $\hat{\mathbf{z}}_i=p$ if and only if $i\in \mathbf S_p\,.$
\end{algorithmic}
\end{algorithm}
\begin{remark}[Running time]
Each step in the outer loop takes time $O(n^2)$. Overall \cref{algorithm:second-moment-rounding} runs in $O(k\cdot n^2)\,.$
\end{remark}

As first step of the proof, we introduce a new definition.

\begin{definition}[Representative row]
Let $z\in [k]^n$, let $c_1(z),\ldots, c_{k}(z)$ be the indicator vectors of the labels in $z$ and let $A^*(z) =\sum_p \dyad{c_p(z)}\,.$ 
For a matrix $A\in \set{0,1}^{n\times n}\,,$ we say $A_i$ is $q$-representative
if 
\begin{align}
    \Snorm{A_i-A^*(z)_i}\leq n\cdot e^{-q\cdot\Cbar^2\cdot t}\,.
\end{align}
We denote by $\cR_{q}$ the set of $q$ representatives of $A\,.$
\end{definition}

The use of representative rows is convenient because, for sufficiently many observations, it is easy to see that rows which are representative of the same community must be close together, while rows that are representative of different communities must be far from each other.

\begin{lemma}\label{lemma:sucess-when-chosing-good-representative}
 Let $n,k,t>0$\,, and let $q>0$ be the hidden constant in \Cref{lemma:structure-of-graph}. Let $t\geq C\frac{\log k}{\Cbar^2}$ for a large enough universal constant $C>0$. Let $z\in [k]^n$ be balanced.
 Suppose that at each iteration of the outer loop, \cref{algorithm:second-moment-rounding} picks some $A_{\mathbf{i}}$ that is a $q$-representative. Then $\max_{\pi\in P_k}\sum_{i\in\cR_q(z)}\bracbb{\hat{\mathbf{z}_i}=\pi(z_i)}=\Card{\cR_q}\,,$ where $P_k$ is the permutation group over $[k]\,.$
\end{lemma}
We defer the proof of \cref{lemma:sucess-when-chosing-good-representative} to \cref{section:deferred-proofs}.
Now, to prove \cref{theorem:main-algorithm}, it remains to argue that, with high probability, \textit{first} there are few non-representatives in $\mathbf A\,,$ and \textit{second} \cref{algorithm:second-moment-rounding}  picks $q$-representatives at each iteration of the outer loop.

\begin{lemma}\label{lemma:probability-pick-representatives}
 Let $n,k,t>0$\,, and let $q>0$ be the hidden constant in \Cref{lemma:structure-of-graph}. Let $t= C\frac{\log k}{\bar{C}^2}$ for a large enough universal constant $C>0$. Let $\bm \cI:=(\mathbf z,(\mathbf f_1,\mathbf G_1),\ldots(\mathbf f_t,\mathbf G_t))\sim \hsbm\,.$ Let $\mathbf{A}$ be the matrix constructed by \cref{algorithm:heterogeneous-models}. On input $\mathbf{A}$, with probability at least $1-k^{-\Omega(1)}$, the following holds:
 \begin{itemize}
     \item there are at least $n\cdot (1-k^{-\Omega(1)})$ rows in $\mathbf A$ which are $\Omega(q)$-representatives,
     \item  step 1.(a) of \cref{algorithm:second-moment-rounding} only picks $\Omega(q)$-representative vectors.
 \end{itemize}
\begin{proof}
By \cref{fact:balanced-instance} we may assume $\mathbf{z}$ is balanced.
By \cref{lemma:structure-of-graph} and Markov's inequality, with probability $1 - e^{-\Omega(q\cdot\Cbar^2\cdot t)}$ there are at most $O(n)\cdot e^{-\Omega\Paren{q\cdot\Cbar^2\cdot t}}$ indices $\mathbf i\in[n]$ such that $\mathbf{A_i}$ is not a $\Omega(q)$-representative. Note that if $C$ is large enough, then $e^{-\Omega\Paren{q\cdot\Cbar^2\cdot t}}=k^{-\Omega(1)}$. 
At every iteration of the loop, if we pick a representative vector we remove at most $(1+o(1))\frac{n}{k}$ indices, since $\mathbf z$ is balanced. Hence at each iteration there are at least $\frac{n}{k}(1-o(1))$ indices that are $\Omega(q)$-representative and which have not yet been picked.
It follows that the probability we never pick an index that is not $\Omega(q)$-representative is at least $\Paren{1-\frac{n\cdot k^{-\Omega(1)}}{n/k}}^k\geq 1-k^{-\Omega(1)}$ as desired (by choosing $C$ to be large enough).
\end{proof}
\end{lemma}

\cref{theorem:main-algorithm} now immediately follows combining \cref{fact:balanced-instance}, \cref{lemma:structure-of-graph}, \cref{lemma:sucess-when-chosing-good-representative} and \cref{lemma:probability-pick-representatives}.

\paragraph{Exact recovery}
\cref{lemma:probability-pick-representatives} also implicitly yield exact recovery for sufficiently many observations.

\begin{corollary}\label{corollary:probability-pick-representatives-exact-recovery}
 Let $n,k,t>0$\,, and let $q>0$ be the hidden constant in \Cref{lemma:structure-of-graph}. Let $t\geq  C\frac{\log n}{\bar{C}^2}$ for a large enough universal constant $C>0$. Let $\bm \cI:=(\mathbf z,(\mathbf f_1,\mathbf G_1),\ldots(\mathbf f_t,\mathbf G_t))\sim \hsbm\,.$ Let $\mathbf{A}$ be the matrix constructed by \cref{algorithm:heterogeneous-models}. On input $\mathbf{A}$, with probability at least $1-n^{-\Omega(1)}$, all rows in $\mathbf A$ are $\Omega(q)$-representatives.
\begin{proof}
By \cref{lemma:structure-of-graph} and Markov's inequality, with probability $1 - e^{-\Omega(q\cdot\Cbar^2\cdot t)}$ there are at most $O(n)\cdot e^{-\Omega\Paren{q\cdot\Cbar^2\cdot t}}$ indices $\mathbf i\in[n]$ such that $\mathbf{A_i}$ is not a $\Omega(q)$-representative. Therefore, for $t \ge \frac{C}{\bar{C}^2}\log n$ no such index exists.
\end{proof}
\end{corollary}

Since by \cref{lemma:sucess-when-chosing-good-representative} only indices corresponding to rows that are not $\Omega(q)$-representative can be misclassified, by  \cref{fact:balanced-instance}, \cref{lemma:structure-of-graph},  \cref{lemma:probability-pick-representatives} and \cref{corollary:probability-pick-representatives-exact-recovery} we obtain \cref{corollary:exact-recovery-main}. 

\section{Experiments}\label{section:experiments}
We show here experiments on synthetic data sampled from $\hsbmt$\,.
The estimator in \cref{theorem:weak-recovery-sbm} is complex and relies on a high order sum-of-squares program, making it hard to implement in practice. Nevertheless, it is reasonable to believe that: \textit{(i)} the guarantees of \cref{theorem:weak-recovery-sbm} are only sufficient, but not necessary, to obtain \cref{theorem:main-algorithm}; \textit{(ii)} other estimators provide the guarantees of \cref{theorem:weak-recovery-sbm}.
For this reason, it  makes sense to  test \cref{algorithm:heterogeneous-models} with other community detection algorithms.

The next figures compares the results on $(\mathbf{z}, (\mathbf{f}_1,\mathbf{G}_1),\ldots,(\mathbf{f}_t, \mathbf{G}_t)) \sim \hsbmt$ (for a wide range of parameters) of the following algorithms:
\begin{enumerate}
    \item[A.1] Louvain's algorithm \cite{blondel2008fast} on the union graph $\bigcup_{i \in[t]}\mathbf{G}_t\,.$
    \item[A.2] \cref{algorithm:heterogeneous-models} with Louvain's algorithm applied in place of the estimator of \cref{theorem:weak-recovery-sbm}\,.
\end{enumerate}
 The $y$-axis measures agreement as defined in \cref{eq:agreement}. Results are averaged over $20$ simulations.
 
\begin{figure}[ht]
    \centering
\begin{minipage}[b]{0.6\textwidth}
    \includegraphics[width=\textwidth]{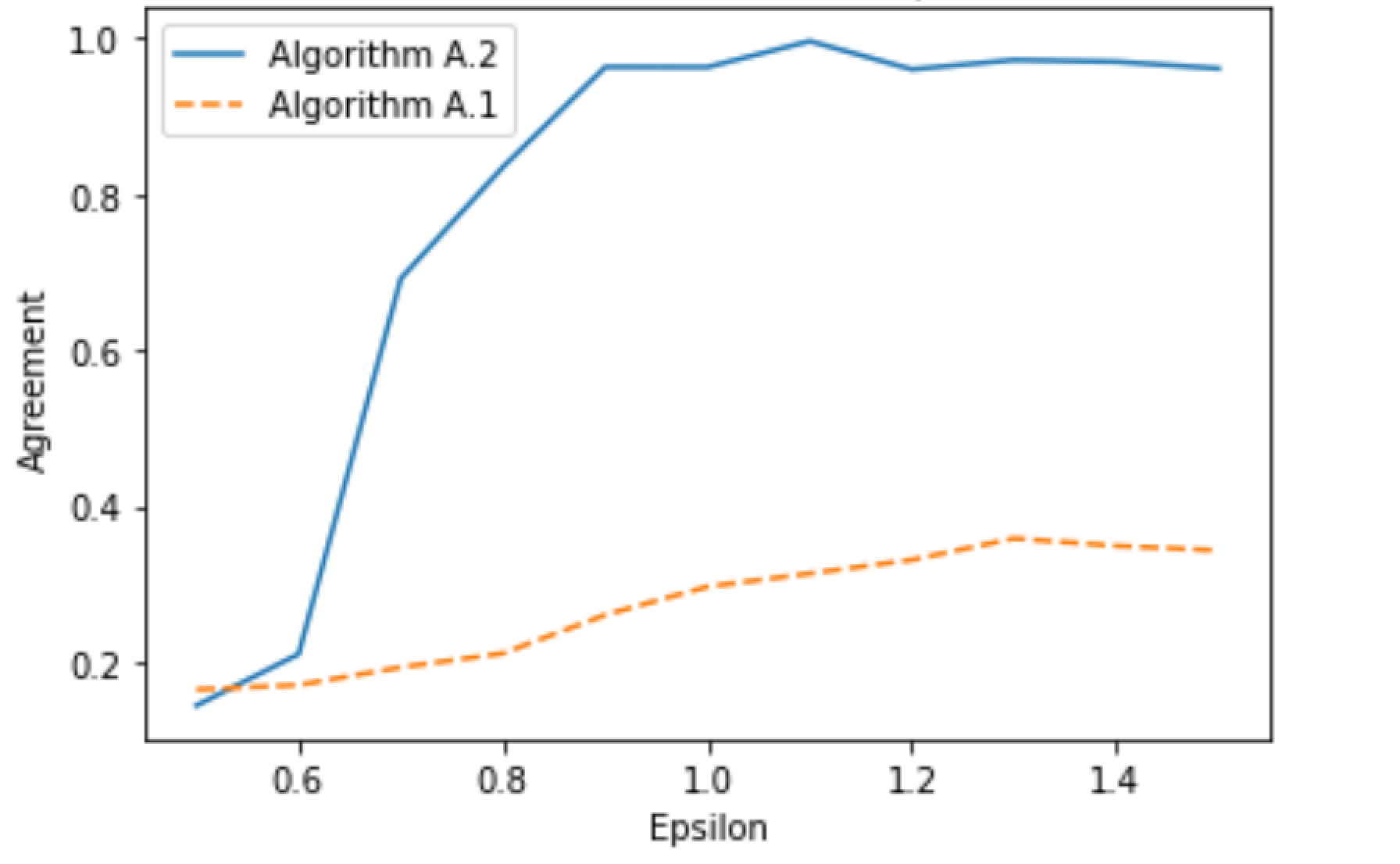}
    \vskip -0.1in
    \caption{Fixing $t=10$, $n=1000$, $k=10$, $d=50$ and varying $\epsilon$ in $[0.5,1.5]$.}
    \label{experiment1}
  \end{minipage}
  \hfill
  \vskip 0.31in
  \begin{minipage}[b]{0.61\textwidth}
    \includegraphics[width=\textwidth]{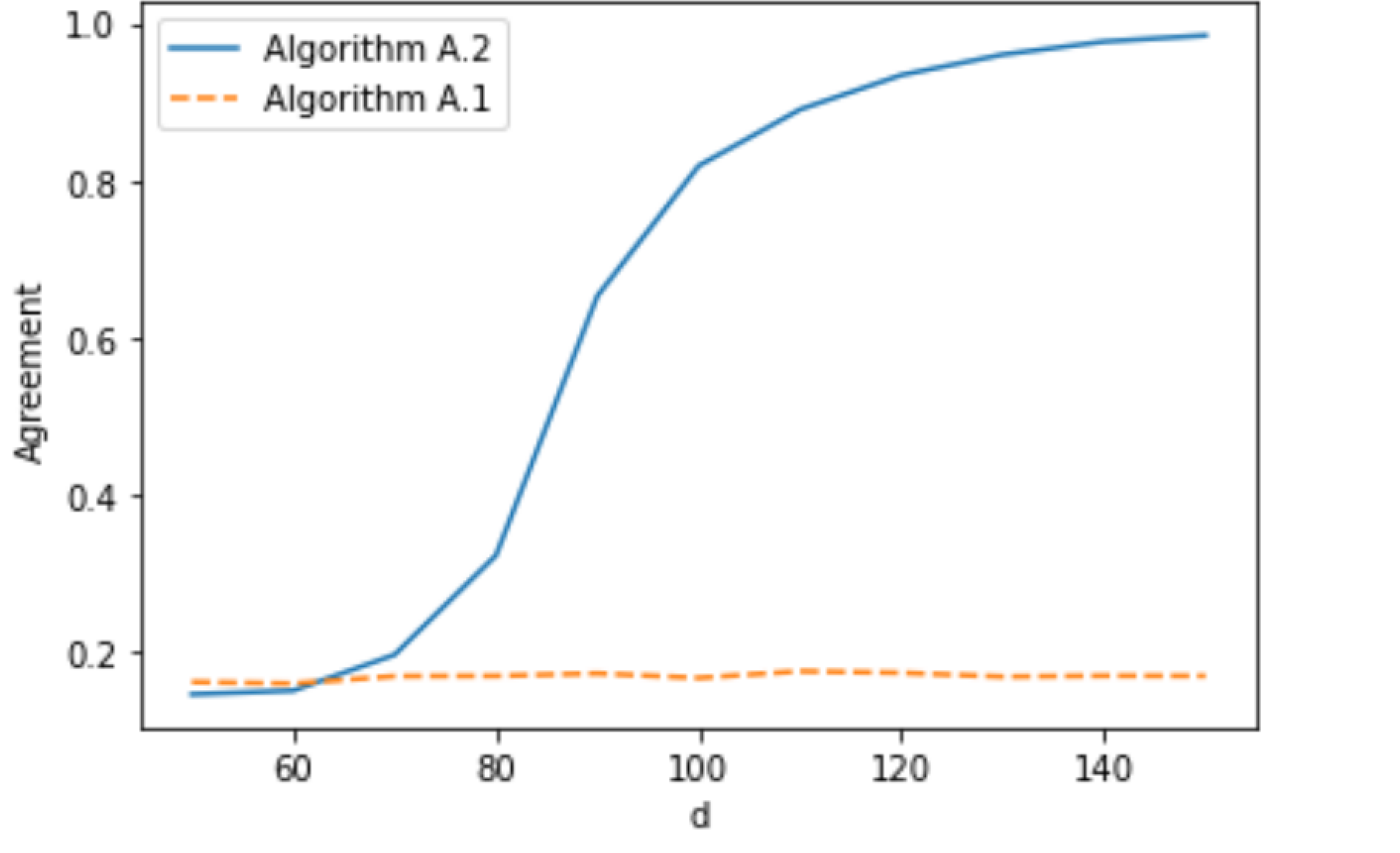}
    \vskip -0.1in
    \caption{Fixing $t=10$, $n=1000$, $k=10$, $\epsilon=0.5$ and varying $d$ in $[50, 150]$.}
    \label{experiment2}
  \end{minipage}
\end{figure}

\section{Conclusions and future directions}\label{section:conclusions}
The introduction of \cref{model:heterogeneous-sbm} raises several natural questions, for which we only provide initial answers.

\paragraph{On the phase transition threshold}
One of the most interesting question concerns the phase transition of the model. Concretely, one may expect a rich interplay between the signal-to-noise ratio of each of the observed graphs (possibly below the relative KS threshold) and the number of observations required to weakly recover the hidden vector $\mathbf{z}$. We leave the characterization of this trade-off beyond \cref{theorem:main-algorithm} and \cref{theorem:lowerbound} as a fascinating open question.

\paragraph{From 2 communities to $k$ communities in the multi-view model}
Another natural question concern the generalization to a model in which each view may have $2\leq k_\ell \leq k$ communities.
The ideas outlined above translate in principle to these settings but the correctness appears difficult to prove.
Concretely, any estimator achieving guarantees comparable to \cref{theorem:weak-recovery-sbm} but for more than $2$ communities can immediately be plugged-in \cref{algorithm:heterogeneous-models} to achieve weak recovery in these more general settings.\footnote{Without changing the proof of correctness of the algorithm!}
However, to the best of our knowledge  existing weak-recovery algorithms for $\sbmk$ do not lead to estimators of this form.

\section*{Acknowledgments}
We thank David Steurer for insightful discussions in the early stages of this work. Tommaso d’Orsi is partially supported by the project MUR FARE2020 PAReCoDi.


\bibliography{custom}

\begin{thebibliography}{40}
\providecommand{\natexlab}[1]{#1}
\providecommand{\url}[1]{\texttt{#1}}
\expandafter\ifx\csname urlstyle\endcsname\relax
  \providecommand{\doi}[1]{doi: #1}\else
  \providecommand{\doi}{doi: \begingroup \urlstyle{rm}\Url}\fi

\bibitem[Abavisani \& Patel(2018)Abavisani and Patel]{abavisani2018deep}
Abavisani, M. and Patel, V.~M.
\newblock Deep multimodal subspace clustering networks.
\newblock \emph{IEEE Journal of Selected Topics in Signal Processing}, 12\penalty0 (6):\penalty0 1601--1614, 2018.

\bibitem[Abbe(2017)]{abbe2017community}
Abbe, E.
\newblock Community detection and stochastic block models: recent developments.
\newblock \emph{The Journal of Machine Learning Research}, 18\penalty0 (1):\penalty0 6446--6531, 2017.

\bibitem[Abbe \& Sandon(2016{\natexlab{a}})Abbe and Sandon]{abbe2016achieving}
Abbe, E. and Sandon, C.
\newblock Achieving the ks threshold in the general stochastic block model with linearized acyclic belief propagation.
\newblock \emph{Advances in Neural Information Processing Systems}, 29, 2016{\natexlab{a}}.

\bibitem[Abbe \& Sandon(2016{\natexlab{b}})Abbe and Sandon]{abbe2016crossing}
Abbe, E. and Sandon, C.
\newblock Crossing the ks threshold in the stochastic block model with information theory.
\newblock In \emph{2016 IEEE International Symposium on Information Theory (ISIT)}, pp.\  840--844. IEEE, 2016{\natexlab{b}}.

\bibitem[Abbe et~al.(2015)Abbe, Bandeira, and Hall]{abbe2015exact}
Abbe, E., Bandeira, A.~S., and Hall, G.
\newblock Exact recovery in the stochastic block model.
\newblock \emph{IEEE Transactions on information theory}, 62\penalty0 (1):\penalty0 471--487, 2015.

\bibitem[Banks et~al.(2016)Banks, Moore, Neeman, and Netrapalli]{banks2016information}
Banks, J., Moore, C., Neeman, J., and Netrapalli, P.
\newblock Information-theoretic thresholds for community detection in sparse networks.
\newblock In \emph{Conference on Learning Theory}, pp.\  383--416. PMLR, 2016.

\bibitem[Bansal et~al.(2004)Bansal, Blum, and Chawla]{bansal2004correlation}
Bansal, N., Blum, A., and Chawla, S.
\newblock Correlation clustering.
\newblock \emph{Machine learning}, 56:\penalty0 89--113, 2004.

\bibitem[Blondel et~al.(2008)Blondel, Guillaume, Lambiotte, and Lefebvre]{blondel2008fast}
Blondel, V.~D., Guillaume, J.-L., Lambiotte, R., and Lefebvre, E.
\newblock Fast unfolding of communities in large networks.
\newblock \emph{Journal of statistical mechanics: theory and experiment}, 2008\penalty0 (10):\penalty0 P10008, 2008.

\bibitem[Bordenave et~al.(2015)Bordenave, Lelarge, and Massouli{\'e}]{bordenave2015non}
Bordenave, C., Lelarge, M., and Massouli{\'e}, L.
\newblock Non-backtracking spectrum of random graphs: community detection and non-regular ramanujan graphs.
\newblock In \emph{2015 IEEE 56th Annual Symposium on Foundations of Computer Science}, pp.\  1347--1357. IEEE, 2015.

\bibitem[Corneli et~al.(2016)Corneli, Latouche, and Rossi]{corneli2016exact}
Corneli, M., Latouche, P., and Rossi, F.
\newblock Exact icl maximization in a non-stationary temporal extension of the stochastic block model for dynamic networks.
\newblock \emph{Neurocomputing}, 192:\penalty0 81--91, 2016.

\bibitem[Dasgupta(2016)]{dasgupta2016cost}
Dasgupta, S.
\newblock A cost function for similarity-based hierarchical clustering.
\newblock In \emph{Proceedings of the forty-eighth annual ACM symposium on Theory of Computing}, pp.\  118--127, 2016.

\bibitem[De~Bacco et~al.(2017)De~Bacco, Power, Larremore, and Moore]{PhysRevE.95.042317}
De~Bacco, C., Power, E.~A., Larremore, D.~B., and Moore, C.
\newblock Community detection, link prediction, and layer interdependence in multilayer networks.
\newblock \emph{Phys. Rev. E}, 95:\penalty0 042317, Apr 2017.
\newblock \doi{10.1103/PhysRevE.95.042317}.
\newblock URL \url{https://link.aps.org/doi/10.1103/PhysRevE.95.042317}.

\bibitem[De~Santiago et~al.(2023)De~Santiago, Szafranski, and Ambroise]{de2023mixture}
De~Santiago, K., Szafranski, M., and Ambroise, C.
\newblock Mixture of stochastic block models for multiview clustering.
\newblock In \emph{ESANN 2023-European Symposium on Artificial Neural Networks, Computational Intelligence and Machine Learning}, pp.\  151--156, 2023.

\bibitem[Decelle et~al.(2011)Decelle, Krzakala, Moore, and Zdeborov{\'a}]{decelle2011asymptotic}
Decelle, A., Krzakala, F., Moore, C., and Zdeborov{\'a}, L.
\newblock Asymptotic analysis of the stochastic block model for modular networks and its algorithmic applications.
\newblock \emph{Physical Review E}, 84\penalty0 (6):\penalty0 066106, 2011.

\bibitem[Ding et~al.(2022)Ding, d'Orsi, Nasser, and Steurer]{ding2022robust}
Ding, J., d'Orsi, T., Nasser, R., and Steurer, D.
\newblock Robust recovery for stochastic block models.
\newblock In \emph{2021 IEEE 62nd Annual Symposium on Foundations of Computer Science (FOCS)}, pp.\  387--394. IEEE, 2022.

\bibitem[Ding et~al.(2023)Ding, d’Orsi, Hua, and Steurer]{ding2023node}
Ding, J., d’Orsi, T., Hua, Y., and Steurer, D.
\newblock Reaching kesten-stigum threshold in the stochastic block model under node corruptions.
\newblock In \emph{The Thirty Sixth Annual Conference on Learning Theory}, pp.\  4044--4071. PMLR, 2023.

\bibitem[Fang et~al.(2023)Fang, Li, Li, Gao, Jia, and Zhang]{fang2023comprehensive}
Fang, U., Li, M., Li, J., Gao, L., Jia, T., and Zhang, Y.
\newblock A comprehensive survey on multi-view clustering.
\newblock \emph{IEEE Transactions on Knowledge and Data Engineering}, 2023.

\bibitem[Fu et~al.(2020)Fu, Lin, Vasilakos, and Wang]{fu2020overview}
Fu, L., Lin, P., Vasilakos, A.~V., and Wang, S.
\newblock An overview of recent multi-view clustering.
\newblock \emph{Neurocomputing}, 402:\penalty0 148--161, 2020.

\bibitem[Goldberg(1984)]{goldberg1984finding}
Goldberg, A.~V.
\newblock Finding a maximum density subgraph.
\newblock 1984.

\bibitem[Gorovits et~al.(2018)Gorovits, Gujral, Papalexakis, and Bogdanov]{gorovits2018larc}
Gorovits, A., Gujral, E., Papalexakis, E.~E., and Bogdanov, P.
\newblock Larc: Learning activity-regularized overlapping communities across time.
\newblock In \emph{Proceedings of the 24th ACM SIGKDD international conference on knowledge discovery \& data mining}, pp.\  1465--1474, 2018.

\bibitem[Gujral \& Papalexakis(2018)Gujral and Papalexakis]{gujral2018smacd}
Gujral, E. and Papalexakis, E.~E.
\newblock Smacd: Semi-supervised multi-aspect community detection.
\newblock In \emph{Proceedings of the 2018 SIAM International Conference on Data Mining}, pp.\  702--710. SIAM, 2018.

\bibitem[Gujral et~al.(2020)Gujral, Pasricha, and Papalexakis]{gujral2020beyond}
Gujral, E., Pasricha, R., and Papalexakis, E.
\newblock Beyond rank-1: Discovering rich community structure in multi-aspect graphs.
\newblock In \emph{Proceedings of The Web Conference 2020}, pp.\  452--462, 2020.

\bibitem[Han et~al.(2015)Han, Xu, and Airoldi]{pmlr-v37-hanb15}
Han, Q., Xu, K., and Airoldi, E.
\newblock Consistent estimation of dynamic and multi-layer block models.
\newblock In Bach, F. and Blei, D. (eds.), \emph{Proceedings of the 32nd International Conference on Machine Learning}, volume~37 of \emph{Proceedings of Machine Learning Research}, pp.\  1511--1520, Lille, France, 07--09 Jul 2015. PMLR.
\newblock URL \url{https://proceedings.mlr.press/v37/hanb15.html}.

\bibitem[Hopkins \& Steurer(2017)Hopkins and Steurer]{hopkins2017efficient}
Hopkins, S.~B. and Steurer, D.
\newblock Efficient bayesian estimation from few samples: community detection and related problems.
\newblock In \emph{2017 IEEE 58th Annual Symposium on Foundations of Computer Science (FOCS)}, pp.\  379--390. IEEE, 2017.

\bibitem[Hu et~al.(2019)Hu, Nie, and Li]{hu2019deep}
Hu, D., Nie, F., and Li, X.
\newblock Deep multimodal clustering for unsupervised audiovisual learning.
\newblock In \emph{Proceedings of the IEEE/CVF Conference on Computer Vision and Pattern Recognition}, pp.\  9248--9257, 2019.

\bibitem[Khan \& Maji(2019)Khan and Maji]{khan2019approximate}
Khan, A. and Maji, P.
\newblock Approximate graph laplacians for multimodal data clustering.
\newblock \emph{IEEE transactions on pattern analysis and machine intelligence}, 43\penalty0 (3):\penalty0 798--813, 2019.

\bibitem[Kim et~al.(2016)Kim, Han, and Ko]{kim2016joint}
Kim, M., Han, D.~K., and Ko, H.
\newblock Joint patch clustering-based dictionary learning for multimodal image fusion.
\newblock \emph{Information Fusion}, 27:\penalty0 198--214, 2016.

\bibitem[Liu et~al.(2022)Liu, Mohanty, and Raghavendra]{liu2022statistical}
Liu, S., Mohanty, S., and Raghavendra, P.
\newblock On statistical inference when fixed points of belief propagation are unstable.
\newblock In \emph{2021 IEEE 62nd Annual Symposium on Foundations of Computer Science (FOCS)}, pp.\  395--405. IEEE, 2022.

\bibitem[Massouli{\'e}(2014)]{massoulie2014community}
Massouli{\'e}, L.
\newblock Community detection thresholds and the weak ramanujan property.
\newblock In \emph{Proceedings of the forty-sixth annual ACM symposium on Theory of computing}, pp.\  694--703, 2014.

\bibitem[Montanari \& Sen(2016)Montanari and Sen]{montanari2016semidefinite}
Montanari, A. and Sen, S.
\newblock Semidefinite programs on sparse random graphs and their application to community detection.
\newblock In \emph{Proceedings of the forty-eighth annual ACM symposium on Theory of Computing}, pp.\  814--827, 2016.

\bibitem[Mossel et~al.(2014)Mossel, Neeman, and Sly]{mossel2014belief}
Mossel, E., Neeman, J., and Sly, A.
\newblock Belief propagation, robust reconstruction and optimal recovery of block models.
\newblock In \emph{Conference on Learning Theory}, pp.\  356--370. PMLR, 2014.

\bibitem[Mossel et~al.(2015{\natexlab{a}})Mossel, Neeman, and Sly]{mossel2015consistency}
Mossel, E., Neeman, J., and Sly, A.
\newblock Consistency thresholds for the planted bisection model.
\newblock In \emph{Proceedings of the forty-seventh annual ACM symposium on Theory of computing}, pp.\  69--75, 2015{\natexlab{a}}.

\bibitem[Mossel et~al.(2015{\natexlab{b}})Mossel, Neeman, and Sly]{mossel2015reconstruction}
Mossel, E., Neeman, J., and Sly, A.
\newblock Reconstruction and estimation in the planted partition model.
\newblock \emph{Probability Theory and Related Fields}, 162:\penalty0 431--461, 2015{\natexlab{b}}.

\bibitem[Mossel et~al.(2018)Mossel, Neeman, and Sly]{mossel2018proof}
Mossel, E., Neeman, J., and Sly, A.
\newblock A proof of the block model threshold conjecture.
\newblock \emph{Combinatorica}, 38\penalty0 (3):\penalty0 665--708, 2018.

\bibitem[Ng et~al.(2001)Ng, Jordan, and Weiss]{ng2001spectral}
Ng, A., Jordan, M., and Weiss, Y.
\newblock On spectral clustering: Analysis and an algorithm.
\newblock \emph{Advances in neural information processing systems}, 14, 2001.

\bibitem[Ni et~al.(2016)Ni, Cheng, Fan, and Zhang]{ni2016self}
Ni, J., Cheng, W., Fan, W., and Zhang, X.
\newblock Self-grouping multi-network clustering.
\newblock In \emph{2016 IEEE 16th International Conference on Data Mining (ICDM)}, pp.\  1119--1124. IEEE, 2016.

\bibitem[Papalexakis et~al.(2013)Papalexakis, Akoglu, and Ience]{papalexakis2013more}
Papalexakis, E.~E., Akoglu, L., and Ience, D.
\newblock Do more views of a graph help? community detection and clustering in multi-graphs.
\newblock In \emph{Proceedings of the 16th International Conference on Information Fusion}, pp.\  899--905. IEEE, 2013.

\bibitem[Paul \& Chen(2016)Paul and Chen]{paul2016consistent}
Paul, S. and Chen, Y.
\newblock Consistent community detection in multi-relational data through restricted multi-layer stochastic blockmodel.
\newblock 2016.

\bibitem[Von~Luxburg(2007)]{von2007tutorial}
Von~Luxburg, U.
\newblock A tutorial on spectral clustering.
\newblock \emph{Statistics and computing}, 17:\penalty0 395--416, 2007.

\bibitem[Zhong \& Pun(2021)Zhong and Pun]{zhong2021latent}
Zhong, G. and Pun, C.-M.
\newblock Latent low-rank graph learning for multimodal clustering.
\newblock In \emph{2021 IEEE 37th International Conference on Data Engineering (ICDE)}, pp.\  492--503. IEEE, 2021.

\end{thebibliography}
\bibliographystyle{icml2024}

\newpage
\appendix

\section{Failure of community detection on the union graph}\label{section:union-graph}
In this section we provide rigorous evidence that efficient algorithm cannot achieve comparable guarantees to \cref{theorem:main-algorithm} by only considering the union graph $\bigcup_{i\in[t]}\mathbf{G}_i$ for $(\mathbf{z},(\mathbf{f_1,\mathbf{G}_1}),\ldots(\mathbf{f}_t,\mathbf{G}_t))\sim\hsbmt$.
Concretely, we prove the following theorem.

\begin{theorem}[Limits of weak recovery from the union graph]\label{theorem:limits-weak-recovery-union}
Let $n,k,d,\eps>0$, assume that $t\geq 100\cdot(\log k)^2$, and let $\bm \cI:=(\mathbf{z},(\mathbf{f_1,\mathbf{G}_1}),\ldots(\mathbf{f_t},\mathbf{G}_t))\sim\hsbmt$.
With probability at least $1-k^{-\Omega(1)}$ over the draws of  $\mathbf{f}_1,\ldots,\mathbf{f}_t$,
the conditional distribution over $(\mathbf{z}\,,\mathbf{G}_1,\ldots,\mathbf{G}_t)$ satisfies:
\begin{itemize}
    \item $\mathbf{z}$ is drawn uniformly at random from $[k]^n$;
    \item each edge $ij$ appears (independently) in $\bigcup_{i\in[t]}\mathbf{G}_i$ with probability at most $\frac{d^*}{n}(1+(1-\tfrac{1}{k})\eps^*)$ if $\mathbf{z}_i=\mathbf{z}_j$ and probability at least $\frac{d^*}{n}(1-\tfrac{\eps}{k})$ otherwise, for some $d^*,\eps^*$ such that
    \begin{align*}
        dt\cdot \frac{1+\frac{\eps}{2}}{1+\Paren{1-\frac{1}{k}}\eps^*}\leq d^*\leq dt\cdot \frac{1+\frac{\eps^*}{1-\eps^*/k}}{1+\Paren{1-\frac{1}{k}}\eps^*}-o(1)\,.
    \end{align*}
\end{itemize}

\end{theorem}

In words, \cref{theorem:limits-weak-recovery-union} shows that the union graph $\bigcup_{i\in[t]}\mathbf{G}_i$ is essentially a $k$-community stochastic block model with parameters $d^*=\Theta(dt)\,,\eps^*=\Theta(\eps)\,.$
As discussed in \cref{section:introduction}, it is conjecturally hard to achieve weak recovery in polynomial time  for $d^*(\eps^*/k)^2\leq 1\,.$
In the context of \cref{theorem:limits-weak-recovery-union}, this implies that the parameters of the distribution of $\bigcup_{i\in[t]}\mathbf{G}_i$ are above the Kesten-Stigum threshold \textit{only} for $d^*(\eps^*/k)^2\geq \Omega(1)$.
That is, at least $t\geq \Omega(k^2)$ observations are required!

Next we prove the theorem.
\begin{proof}[Proof of \cref{theorem:limits-weak-recovery-union}]
Let $q,q'\in[k]$ be distinct. By Chernoff's bound and choice of $t$ we have\footnote{We can take $o(1)$ to be $t^{-1/4}$ and by Hoeffding's inequality we can bound the probability of the event $\cE$ in \eqref{eq:chernoff_bound_on_event_on_typical_f} not happening by $2\exp(-\sqrt{t})\leq 2\exp(-10\log k)\leq k^{-5}$.}
\begin{equation}
\label{eq:chernoff_bound_on_event_on_typical_f}
\begin{aligned}
    \bbP\Paren{\frac{t}{2}(1-o(1)) \leq \sum_{\ell\in[t]} \bracbb{\mathbf{f}_\ell(q)=\mathbf{f}_\ell(q')}\leq \frac{t}{2}(1+o(1))}\geq 1 - k^{-5}\,.
\end{aligned}
\end{equation}
Hence we may take a union over all such pairs $q,q'\in[k]$ as the corresponding event $\cE$ will hold with probability at least $1-k^{-O(1)}\,.$
So let $f_1\,\ldots,f_t:[k]\rightarrow{\pm 1}$ be fixed functions verifying the event $\cE$ of \eqref{eq:chernoff_bound_on_event_on_typical_f}.
We condition the rest of the analysis on $\mathbf{f}_1=f_1,\ldots,\mathbf{f}_t=f_t\,.$
In these settings each edge appears in  $\mathbf{G^*} := \bigcup_{i\in[t]}\mathbf{G}_i$ independently of others. Moreover, by union bound
\begin{align*}
    \bbP\Paren{ij\in \mathbf{G^*}\given \mathbf{z}_i=\mathbf{z}_j\,, \mathbf{f}_1=f_1,\ldots,\mathbf{f}_t=f_t} \leq \Paren{1+\frac{\eps}{2}}\frac{dt}{n}\,,
\end{align*}
and 
\begin{align*}
    \bbP&\Paren{ij\in \mathbf{G^*}\given \mathbf{z}_i\neq \mathbf{z}_j\,, \mathbf{f}_1=f_1,\ldots,\mathbf{f}_t=f_t} \\
    &\geq 1-\Paren{1-\Paren{1+\frac{\eps}{2}}\frac{d}{n}}^{\frac{t}{2}(1-o(1))}\Paren{1-\Paren{1-\frac{\eps}{2}}\frac{d}{n}}^{\frac{t}{2}(1+o(1))}\\
    &\geq 1 - \Paren{1-\Paren{1+\frac{\eps}{2}}\frac{dt}{2n}(1-o(1))}\Paren{1-\Paren{1-\frac{\eps}{2}}\frac{dt}{2n}(1+o(1))}\\
    &\geq (1-o(1))\frac{dt}{n}\,,
\end{align*}
where  we used the inequality $(1+s)^r\geq 1+rs\,,$ for $r>1\,, s>-1\,.$
It remains to compute $d^*$ and $\eps^*$ so that 
\begin{align*}
    \Paren{1+\frac{\eps}{2}}\frac{dt}{n} &\leq \frac{d^*}{n}\Paren{1+\Paren{1-\frac{1}{k}}\eps^*}\,,\\
    (1-o(1))\frac{dt}{n}&\geq \frac{d^*}{n}\Paren{1-\frac{\eps^*}{k}}\,.
\end{align*}
Rearranging the inequalities,
\begin{align*}
    d^*&\geq dt\cdot \frac{1+\frac{\eps}{2}}{1+\Paren{1-\frac{1}{k}}\eps^*}\,,\\
    d^*&\leq dt\cdot\frac{1}{1-\frac{\eps^*}{k}}-o(1)=dt\cdot \frac{1+\frac{\eps^*k}{k-\eps^*}}{1+\Paren{1-\frac{1}{k}}\eps^*}-o(1)
\end{align*}
as desired.
\end{proof}

\begin{remark}[On the weighted union graph]
A natural question to ask is whether the weighted union graph -- the graph over $[n]$, in which edge $ij$ has weight $\sum_{\ell\in[t]}\bracbb{\ij\in\mathbf{G}_\ell}$-- could provide better guarantees.
In the sparse settings $d\leq n^{o(1)}\,, t\leq n^{o(1)}$ only a $n^{o(1)-1}$ fraction of the edges have weight larger than $1$ and thus one may expect that this additional information does not simplify the problem.
\end{remark}

\section{Information theoretic lower bound for blackbox algorithms}\label{section:lowerbound}

In this section we would like to study, in the context of $\hsbm$, the information-theoretic limitations for having an algorithm that \textit{(i)} runs the procedure in \cref{theorem:weak-recovery-sbm} on each observation and \textit{(ii)} uses the resulting matrices $(\hat{\mathbf{X}}_\ell)_{\ell\in[t]}$ to reconstruct the original $k$ communities. Essentially, we would like to prove a formal version of \cref{theorem:lowerbound}. In order to do this, we first need to introduce some useful notation and terminology. 

Throughout the section let $\mathbf{z}\in[k]^n$ be the vector of communities, and let $(\mathbf{f}_\ell)_{\ell\in[t]}$ be $t$ independent and uniformly distributed random mappings $[k]\to\{+1,-1\}$. For every $\ell\in[t]$ let $\mathbf{x}_\ell=\mathbf{f}_\ell(\mathbf{z})$ and let $\mathbf{G}_\ell\sim \SBM_{n, 2, d_\ell,\eps_\ell}(\mathbf{f}_\ell(\mathbf{z}))$. We introduce a quantitative version of weak-recovery.

\begin{definition}[$\alpha_\ell$-weak-recovery algorithm]
We say that an algorithm\footnote{Note that $\hat{X}_\ell$ may be a randomized algorithm.} $\hat{X}_\ell$ taking $\mathbf{G}_\ell$ as input and producing an estimate $\hat{X}_\ell(\mathbf{G}_\ell)\in[+1,-1]^{n\times n}$ of $\dyad{\mathbf{x}_\ell}$ is an $\alpha_\ell$-weak-recovery algorithm if we have
\begin{equation}
\label{eq:normalized_agreement_weak_recovery}
    \E\Brac{\hat{X}_\ell(\mathbf{G})_{\ij}\suchthat (\mathbf{x}_\ell)_i=(\mathbf{x}_\ell)_j} - \E\Brac{\hat{X}(\mathbf{G})_{\ij}\suchthat (\mathbf{x}_\ell)_i\neq (\mathbf{x}_\ell)_j}\geq \alpha_\ell\,,\quad\forall i,j\in[n]\,.
\end{equation}
\end{definition}

Clearly, the algorithm mentioned in \cref{theorem:weak-recovery-sbm} is a $C_{d,\epsilon}$-weak-recovery algorithm.

We are interested in determining the information-theoretic limits for estimating $\mathbf{z}$ based only on the outputs of an $\alpha_\ell$-weak-recovery algorithm when applied on the observations $(\mathbf{G}_\ell)_{\ell\in[t]}$. To this end, let us introduce blackbox estimators:

\begin{definition}[Blackbox estimator]\label{definition:blackbox-estimator}
A \emph{blackbox estimator} for $\mathbf{z}$ is a mapping $$\hat{z}:\paren{[+1,-1]^{n\times n}}^t\to[k]^n\,.$$
\end{definition}
The blackbox estimator is applied as follows: For every $\ell\in[t]\,,$ we first compute $\hat{\mathbf{X}}_\ell=\hat{X}_\ell(\mathbf{G}_\ell)\,,$ for some $\alpha$-weak-recovery algorithm $\hat{X}_\ell$ for $\mathbf{x}_\ell$ for which we do not know anything about except that it is an $\alpha$-weak-recovery algorithm, and then compute $\hat{\mathbf{z}}=\hat{z}(\hat{\mathbf{X}}_1,\ldots,\hat{\mathbf{X}}_t)$.

We would like to guarantee that the blackbox estimator yields a successful weak-recovery of $\mathbf{z}$ using only the fact that $\hat{\mathbf{X}}_\ell=\hat{X}_\ell(\mathbf{G}_\ell)$ satisfies \eqref{eq:normalized_agreement_weak_recovery} for every $\ell\in[t]$. In order to formalize this, we will use the notion of $\alpha$-estimates:

\begin{definition}[$\alpha$-estimates]
\label{definition:alpha-estimates}
Let $\alpha=(\alpha_\ell)_{\ell\in[t]}\in(0,2]^t$ be a sequence of $t$ positive numbers. Let $\hat{\mathbf{X}}_1,\ldots,\hat{\mathbf{X}}_t\in[+1,-1]^{n\times n}$ be $t$ random matrices. We say that $(\hat{\mathbf{X}}_\ell)_{\ell\in[t]}$ are \emph{$\alpha$-estimates} of $(\dyad{\mathbf{x}_\ell})_{\ell\in[t]}$ if they satisfy the following three conditions:
\begin{itemize}
    \item[(a)] Given $(\mathbf{x}_\ell)_{\ell\in[t]}$, the random matrices $(\hat{\mathbf{X}}_\ell)_{\ell\in[t]}$ are conditionally independent from $\mathbf{z}$.
    \item[(b)] For every $\ell\in [t]$, given $\mathbf{x}_\ell$, the random matrix $\hat{\mathbf{X}}_\ell$ is conditionally independent from $(\mathbf{x}_1,\ldots,\mathbf{x}_{\ell-1},\mathbf{x}_{\ell+1},\ldots,\mathbf{x}_t,\hat{\mathbf{X}}_1,\ldots,\hat{\mathbf{X}}_{\ell-1},\hat{\mathbf{X}}_{\ell+1},\ldots,\hat{\mathbf{X}}_t)$.
    \item[(c)] For every $\ell\in [t]$ and every $i,j\in[n]$ with $i\neq j$ we have
\begin{equation}
    \E\Brac{(\hat{\mathbf{X}}_\ell)_{\ij}\suchthat (\mathbf{x}_\ell)_i=(\mathbf{x}_\ell)_j} - \E\Brac{(\hat{\mathbf{X}}_\ell)_{ij}\suchthat (\mathbf{x}_\ell)_i\neq (\mathbf{x}_\ell)_j}\geq \alpha_\ell\,,\quad\forall i,j\in[n]\,.
\end{equation}
\end{itemize}
\end{definition}

It is not hard to see that if $\hat{X}_\ell$ is an $\alpha_\ell$-weak-recovery algorithms for every $\ell\in[t]$, then $(\hat{X}(\mathbf{G}_\ell))_{\ell\in[t]}$ are $\alpha$-estimates of $(\dyad{\mathbf{x}_\ell})_{\ell\in[t]}$.

Now we are ready to formally define what we mean by ``guaranteeing that the blackbox estimator yields a successful weak-recovery of $\mathbf{z}$ using only the fact that $\hat{\mathbf{X}}_\ell=\hat{X}(\mathbf{G}_\ell)$ satisfies \eqref{eq:normalized_agreement_weak_recovery} for every $\ell\in[t]$":

\begin{definition}[$(\rho,\tau,\alpha,t)$-weak-recovery blackbox estimator]
Let $\alpha=(\alpha_\ell)_{\ell\in[t]}\in(0,2]^t$ be a sequence of $t$ positive numbers, and let $\rho>0$ and $\tau>0\,.$

A mapping\footnote{Note that $\hat{z}$ may be a randomized function.} $$\hat{z}:\paren{[-1,+1]^{n\times n}}^t\to[k]^n$$ is said to be a \emph{$(\rho,\tau,\alpha,t)$-weak-recovery blackbox estimator for $\mathbf{z}$ from $\alpha$-estimates of $(\dyad{\mathbf{x}_\ell})_{\ell\in[t]}$} if for every $t$ random matrices $(\hat{\mathbf{X}}_\ell)_{\ell\in[t]}$ which are \emph{$\alpha$-estimates} of $(\dyad{\mathbf{x}_\ell})_{\ell\in[t]}$, if
$$\mathbf{z}=\hat{z}(\hat{\mathbf{X}}_1,\ldots,\hat{\mathbf{X}}_t)\,,$$
then with probability at least $\tau\,$ we have
\begin{align}\label{eq:heterogeneous-weak-recovery-rho}
    \max_{\pi \in P_k}\sum_i\frac{1}{n}
    \bracbb{\mathbf{z}_i=\pi(\hat{\mathbf{z}}_{i})}\geq \frac{1}{k}+\rho\,,
\end{align}
where $P_k$ is the set of permutations $[k]\to[k]\,.$
\end{definition}

We are now ready to state the main theorem of the section, which implies \cref{theorem:lowerbound}.

\begin{theorem}[Formal statement of \cref{theorem:lowerbound}]\label{theorem:lowerbound-formal}
Let $\alpha=(\alpha_\ell)_{\ell\in[t]}\in(0,2]^t$ be a sequence of $t$ positive numbers and denote
$$\overline{\alpha}=\frac{1}{t}\sum_{\ell\in[t]}\alpha_\ell\,.$$
Let $\rho>0$ and $\tau>0\,.$ Let $\mathbf{z}\in[k]^n$ be the uniformly random vector of communities, and let $(\mathbf{f}_\ell)_{\ell\in[t]}$ be $t$ independent and uniformly distributed random mappings $[k]\to\{+1,-1\}$. For every $\ell\in[t]$ let $\mathbf{x}_\ell=\mathbf{f}_\ell(\mathbf{z})$. If there exists a $(\rho,\tau,\alpha,t)$-weak-recovery blackbox estimator for $\mathbf{z}$ from $\alpha$-estimates of $(\dyad{\mathbf{x}_\ell})_{\ell\in[t]}\,,$ and if $n$ is large enoug, then we must have
$$t\geq \Omega_\rho\Paren{ \frac{\tau\cdot \log k}{\overline{\alpha}}}\,.$$
\end{theorem}

We prove \cref{theorem:lowerbound-formal} in \cref{section:information-revealed-by-estimates}, \cref{section:weak-recovery-reduces-entropy}, and \cref{section:putting-everything-together}.
We conclude this section showing how \cref{algorithm:heterogeneous-models} is indeed a blackbox estimator.

\begin{proof}[Proof that \cref{algorithm:heterogeneous-models} is a blackbox estimator]
We can split \cref{algorithm:heterogeneous-models} into two steps:
\begin{itemize}
    \item[(1)] Computing $\hat{\mathbf{X}}_1=\hat{\mathbf{X}}_1(\mathbf{G}_1),\ldots, \hat{\mathbf{X}}_t=\hat{\mathbf{X}}_t(\mathbf{G}_t)$ by applying the algorithm in \cref{theorem:weak-recovery-sbm} to $\mathbf{G}_1,\ldots,\mathbf{G}_t\,,$ respectively.
    \item[(2)] Computing an estimate $\hat{\mathbf{z}}$ of $\mathbf{z}$ based only on $(\hat{\mathbf{X}}_\ell)_{\ell\in[t]}\,.$
\end{itemize}
In step (1), since each of $\hat{\mathbf{X}}_\ell$ is applied to $\mathbf{G}_\ell$ independently of all other graphs and since $\mathbf{G}_\ell$ depends on $\mathbf{z}$ only through $\mathbf{x}_\ell\,,$ it is not hard to see that $(\hat{\mathbf{X}}_\ell)_{\ell\in[t]}$ satisfy conditions (a) and (b)  of \cref{definition:alpha-estimates}. Now if $\alpha_\ell=C_\ell$ is the correlation guaranteed by \cref{theorem:weak-recovery-sbm} for $\hat{\mathbf{X}_\ell}\,,$ it follows that $(\hat{\mathbf{X}}_\ell)_{\ell\in[t]}$ are $\alpha$-estimates of $(\dyad{\mathbf{x}_\ell})_{\ell\in[t]}\,,$ where $\alpha=(\alpha_\ell)_{\ell\in[t]}\,.$

Now since step (2) of \cref{algorithm:heterogeneous-models} only processes $(\hat{\mathbf{X}}_\ell)_{\ell\in[t]}\,,$ and since the guarantee on the agreement of $\hat{\mathbf{z}}$ with $\mathbf{z}$ is proved based only on the fact that $(\hat{\mathbf{X}}_\ell)_{\ell\in[t]}$ are $\alpha$-estimates, we can see that, assuming that $t$ is a large enough multiple of $(\log k)/\overline{\alpha}^2\,,$ step (2) is a $(1-k^{-\Omega(1)},1-k^{-\Omega(1)},\alpha,t)$-weak-recovery blackbox estimator.
\end{proof}


\subsection{Upper bound on the information revealed by $\alpha$-estimates}\label{section:information-revealed-by-estimates}
The first step in our proof is to determine how much information about $\mathbf{z}$ the $\alpha$-estimates can reveal. Let $(\hat{\mathbf{X}}_\ell)_{\ell\in[t]}$ be $\alpha$-estimates of $(\dyad{\mathbf{x}_\ell})_{\ell\in[t]}$. The mutual information (measured in bits) between $\mathbf{z}$ and $\hat{\mathbf{X}}_1,\ldots,\hat{\mathbf{X}}_{t}$ can be upper bounded as follows:

\begin{equation}
    \label{eq:mutualinfo_upperbound_dataproc}
    \begin{aligned}
    I(\mathbf{z};\hat{\mathbf{X}}_1,\ldots,\hat{\mathbf{X}}_{t})&\stackrel{(\ast)}\leq I(\mathbf{x}_1,\ldots,\mathbf{x}_t;\hat{\mathbf{X}}_1,\ldots,\hat{\mathbf{X}}_{t})\\
    &=H(\hat{\mathbf{X}}_1,\ldots,\hat{\mathbf{X}}_{t}) - H(\hat{\mathbf{X}}_1,\ldots,\hat{\mathbf{X}}_{t} | \mathbf{x}_1,\ldots,\mathbf{x}_t)\,,
\end{aligned}
\end{equation}
where $(\ast)$ follows from the data-processing inequality\footnote{Notice that due to Property (a) of $\alpha$-estimates, $\mathbf{z}-(\mathbf{x}_\ell)_{\ell\in[t]}-(\hat{\mathbf{x}}_\ell)_{\ell\in[t]}$ is a Markov chain.}.

The entropy $H(\hat{\mathbf{X}}_1,\ldots,\hat{\mathbf{X}}_{t})$ can be upper bounded as follows:
\begin{equation}
    \label{eq:entropy_estimates_upperbound}
    \begin{aligned}
    H(\hat{\mathbf{X}}_1,\ldots,\hat{\mathbf{X}}_{t})\leq \sum_{\ell\in[t]}H(\hat{\mathbf{X}}_\ell)\,.
\end{aligned}
\end{equation}
Now using the chain rule, the conditional entropy $H(\hat{\mathbf{X}}_1,\ldots,\hat{\mathbf{X}}_{t} | \mathbf{x}_1,\ldots,\mathbf{x}_t)$ can be rewritten as follows:
\begin{equation}
    \label{eq:conditional_entropy_estimates}
    \begin{aligned}
     H(\hat{\mathbf{X}}_1,\ldots,\hat{\mathbf{X}}_{t} | \mathbf{x}_1,\ldots,\mathbf{x}_t) &= \sum_{\ell\in[t]} H(\hat{\mathbf{X}}_\ell | \mathbf{x}_1,\ldots,\mathbf{x}_t,\hat{\mathbf{X}}_1,\ldots,\hat{\mathbf{X}}_{\ell-1})\\
     &=\sum_{\ell\in[t]} H(\hat{\mathbf{X}}_\ell |\mathbf{x}_\ell)\,,
\end{aligned}
\end{equation}
where the last equality follows from Property (b) of $\alpha$-estimates.

Combining \eqref{eq:mutualinfo_upperbound_dataproc}, \eqref{eq:entropy_estimates_upperbound} and \eqref{eq:conditional_entropy_estimates} we get

\begin{equation}
    \label{eq:mutualinfo_upperbound_sum}
    \begin{aligned}
    I(\mathbf{z};\hat{\mathbf{X}}_1,\ldots,\hat{\mathbf{X}}_{t})
    &\leq \sum_{l\in[t]} \Paren{H(\hat{\mathbf{X}}_\ell) - H(\hat{\mathbf{X}}_\ell |\mathbf{x}_\ell)}\\
    &= \sum_{l\in[t]} I(\mathbf{x}_\ell;\hat{\mathbf{X}}_\ell)\,
\end{aligned}
\end{equation}

Now for each $\ell\in[t]$, we will derive an upper bound on $I(\mathbf{x}_\ell;\hat{\mathbf{X}}_\ell)$ (which would then induce an upper bound on $I(\mathbf{z};\hat{\mathbf{X}}_1,\ldots,\hat{\mathbf{X}}_{t})$). Note that we cannot obtain a non-trivial upper bound on $I(\mathbf{x}_\ell;\hat{\mathbf{X}}_\ell)$ for arbitrary $\alpha$-estimates because setting $\hat{\mathbf{X}}_\ell=\dyad{\mathbf{x}_\ell}$ would satisfy the definition of $\alpha$-estimates, and for $\hat{\mathbf{X}}_\ell=\dyad{\mathbf{x}_\ell}$ we have $I(\mathbf{x}_\ell;\hat{\mathbf{X}}_\ell)=n\,,$ which is too large for our purposes. What we will do instead is to show that there exist $\alpha$-estimates for which we can get the desired upper bound on $I(\mathbf{x}_\ell;\hat{\mathbf{X}}_\ell)$.

The $\alpha$-estimates that we will consider are of the form $\hat{\mathbf{X}}_\ell=\dyad{\hat{\mathbf{x}}_\ell}\,$ where $\hat{\mathbf{x}}_\ell\in\{+1,1\}^n$ is defined as follows:

\begin{align*}
    \bbP\Paren{\hat{\mathbf{x}}_\ell=\hat{x}_\ell|\mathbf{x}_\ell=x_\ell}=\prod_{i\in[n]}\bbP\Paren{(\hat{\mathbf{x}}_\ell)_i=(\hat{x}_\ell)_i|(\mathbf{x}_\ell)_i=(x_\ell)_i}\,,
\end{align*}
where
\begin{align*}
    \bbP\Paren{(\hat{\mathbf{x}}_\ell)_i=(\hat{x}_\ell)_i\given(\mathbf{x}_\ell)_i=(x_\ell)_i}=
    \begin{cases}
    \frac{1}{2}+\sqrt{\frac{\alpha_\ell}{8}}\quad&\text{if }(\hat{x}_\ell)_i=(x_\ell)_i\,,\\
    \frac{1}{2}-\sqrt{\frac{\alpha_\ell}{8}}\quad&\text{if }(\hat{x}_\ell)_i=-(x_\ell)_i\,.
    \end{cases}
\end{align*}
In other words, we obtain $\hat{\mathbf{x}}_\ell$ by sending the entries of $\mathbf{x}_\ell$ through a binary symmetric channel with flipping probability $\frac{1}{2}-\sqrt{\frac{\alpha_\ell}{8}}$.

Now notice that
\begin{align*}
    \E\Brac{(\hat{\mathbf{x}}_\ell)_i\given(\mathbf{x}_\ell)_i }&=
    (\mathbf{x}_\ell)_i\cdot\bbP\Paren{(\hat{\mathbf{x}}_\ell)_i=(\mathbf{x}_\ell)_i \given (\mathbf{x}_\ell)_i}-(\mathbf{x}_\ell)_i\cdot\bbP\Paren{(\hat{\mathbf{x}}_\ell)_i=-(\mathbf{x}_\ell)_i\given(\mathbf{x}_\ell)_i}\\
    &=2\sqrt{\frac{\alpha_\ell}{8}}(\mathbf{x}_\ell)_i=\sqrt{\frac{\alpha_\ell}{2}}(\mathbf{x}_\ell)_i\,.
\end{align*}
Hence, for $i\neq j$
\begin{align*}
    \E\Brac{(\hat{\mathbf{X}}_\ell)_{i,j}\given(\mathbf{x}_\ell)_i,(\mathbf{x}_\ell)_j }&= \E\Brac{(\hat{\mathbf{x}}_\ell)_i\cdot(\hat{\mathbf{x}}_\ell)_j \given(\mathbf{x}_\ell)_i,(\mathbf{x}_\ell)_j }\\
    &=\E\Brac{(\hat{\mathbf{x}}_\ell)_i\given(\mathbf{x}_\ell)_i }\E\Brac{(\hat{\mathbf{x}}_\ell)_j\given(\mathbf{x}_\ell)_j }\\
    &= \sqrt{\frac{\alpha_\ell}{2}}(\mathbf{x}_\ell)_i\cdot \sqrt{\frac{\alpha_\ell}{2}}(\mathbf{x}_\ell)_j=\frac{\alpha_\ell}{2}(\mathbf{x}_\ell)_i\cdot(\mathbf{x}_\ell)_j\,,
\end{align*}
from which it is not hard to see that
\begin{align*}
    \E\Brac{(\hat{\mathbf{X}}_\ell)_{i,j}\given(\mathbf{x}_\ell)_i=(\mathbf{x}_\ell)_j }-\E\Brac{(\hat{\mathbf{X}}_\ell)_{i,j}\given(\mathbf{x}_\ell)_i=-(\mathbf{x}_\ell)_j }&=\frac{\alpha_\ell}{2}-\Paren{-\frac{\alpha_\ell}{2}}=\alpha_\ell\,.
\end{align*}

This proves that our choice of $(\hat{\mathbf{X}}_\ell)_{\ell\in[t]}$ indeed yields $\alpha$-estimates. In the remainder of this subsection we will show that this particular choice of $\alpha$-estimates is noisy enough to yield a useful upper bound on the mutual information $I(\mathbf{x}_\ell;\hat{\mathbf{X}}_\ell)\,.$

For every $\ell\in[t]$ we have
\begin{equation}
\label{eq:individual-mutual-info-upper-bound}
    I(\mathbf{x}_\ell;\hat{\mathbf{X}}_\ell) = 
    I(\mathbf{x}_\ell;\hat{\mathbf{x}}_\ell)=H(\hat{\mathbf{x}}_\ell)-H(\hat{\mathbf{x}}_\ell|\mathbf{x}_\ell)\leq n - H(\hat{\mathbf{x}}_\ell|\mathbf{x}_\ell)\,,
\end{equation}
where the first equality follows from the fact that there is a one-to-one mapping between $\hat{\mathbf{x}}_\ell$ and $\hat{\mathbf{X}}_\ell=\dyad{\hat{\mathbf{x}}_\ell}\,,$ and the last inequality follows from the fact that $\hat{\mathbf{x}}_\ell\in\{+1,-1\}^n$ is a binary vector of length $n\,.$

Now notice that the conditional distribution of $\hat{\mathbf{x}}_\ell$ given $\mathbf{x}_\ell$ can be seen as a sequence of $n$ independent Bernoulli random variables with parameter $\frac{1}{2}\pm\sqrt{\frac{\alpha_\ell}{8}}$. Therefore\footnote{Notice that $h_2(p)=h_2(1-p)$ and hence $h_2\Paren{\frac{1}{2}+\sqrt{\frac{\alpha_\ell}{8}}}=h_2\Paren{\frac{1}{2}-\sqrt{\frac{\alpha_\ell}{8}}}$.},
\begin{equation}
\label{eq:conditional-entropy-bsc}
    H(\hat{\mathbf{x}}_\ell|\mathbf{x}_\ell)=n\cdot h_2\Paren{\frac{1}{2}+\sqrt{\frac{\alpha_\ell}{8}}}\,,
\end{equation}
where
$$h_2(p)=-p\log_2 p - (1-p)\log_2(1-p)$$
is the binary entropy function.

Combining \eqref{eq:individual-mutual-info-upper-bound} and \eqref{eq:conditional-entropy-bsc}, we get
\begin{align*}
    I(\mathbf{x}_\ell;\hat{\mathbf{X}}_\ell) \leq n \cdot\Paren{1-h_2\Paren{\frac{1}{2}+\sqrt{\frac{\alpha_\ell}8}}}\,.
\end{align*}

Now note that the function $p\mapsto h_2(p)$ is a strictly concave function achieving its maximum at $p=\frac{1}{2}$ for which we have $h_2(p)=1$. Therefore, for small $\alpha_\ell\,,$ we have
$$1-h_2\Paren{\frac{1}{2}+\sqrt{\frac{\alpha_\ell}8}}=\frac{h_2''(1/2)}{2}\Paren{\sqrt{\frac{\alpha_\ell}8}}^2\pm O\Paren{\sqrt{\frac{\alpha_\ell}8}}^3\leq O({\alpha_\ell})\,.$$
We conclude that there exists an absolute constant $C>0$ such that
\begin{align*}
    I(\mathbf{x}_\ell;\hat{\mathbf{X}}_\ell) \leq C\cdot\alpha_\ell\cdot n \,.
\end{align*}

Combining this with \eqref{eq:mutualinfo_upperbound_sum}, we conclude that for some $\alpha$-estimates $(\hat{\mathbf{X}}_{\ell})_{\ell\in[t]}$ of $(\dyad{\mathbf{x}}_{\ell})_{\ell\in[t]}$, we have
\begin{equation}
\label{eq:mutual_information_upper_bound}
I(\mathbf{z};\hat{\mathbf{X}}_1,\ldots,\hat{\mathbf{X}}_{t})\leq \sum_{\ell\in [t]} C\cdot\alpha_\ell\cdot n = C\cdot \overline{\alpha}\cdot t\cdot n\,.
\end{equation}

\subsection{Weakly recovering $\mathbf{z}$ reduces its entropy}\label{section:weak-recovery-reduces-entropy}

Now let $\hat{\mathbf{z}}\in[k]^n$ be an estimate of $\mathbf{z}$ which satisfies \eqref{eq:heterogeneous-weak-recovery-rho} with probability at least $\tau$. We will apply a modified version of the standard Fano inequality in order to upper bound $H(\mathbf{z}|\hat{\mathbf{z}})$. Define the random variable
$$\mathbf{A}=\begin{cases}
1\quad&\text{if }\mathbf{z}\text{ and }\hat{\mathbf{z}}\text{ satisfy  \eqref{eq:heterogeneous-weak-recovery-rho}}\,,\\
0\quad&\text{otherwise\,.}
\end{cases}$$

We have
\begin{align*}
    H(\mathbf{z}|\hat{\mathbf{z}}) &\leq H(\mathbf{A} , \mathbf{z}|\hat{\mathbf{z}})\\
    &= H(\mathbf{A}|\hat{\mathbf{z}}) +  H(\mathbf{z}|\hat{\mathbf{z}}, \mathbf{A})\\
    &\leq H(\mathbf{A}) +  H(\mathbf{z}|\hat{\mathbf{z}}, \mathbf{A}=0)\bbP[\mathbf{A}=0] +  H(\mathbf{z}|\hat{\mathbf{z}}, \mathbf{A}=1)\bbP[\mathbf{A}=1]\\
    &\leq 1 + (n\log_2k)\cdot\bbP[\mathbf{A}=0] + H(\mathbf{z}|\hat{\mathbf{z}}, \mathbf{A}=1)\cdot\bbP[\mathbf{A}=1]\,,
\end{align*}
where the last inequality follows from the fact that $\mathbf{A}$ is a binary random variable (and hence its entropy is at most one bit), and the fact that $\mathbf{z}\in[k]^n$, which implies that $H(\mathbf{z}|\hat{\mathbf{z}}, \mathbf{A}=0)\leq \log_2\Paren{k^n}=n\log_2 k$. We conclude that
\begin{align*}
    H(\mathbf{z}|\hat{\mathbf{z}}) & \leq 1 + n\log_2k - \Paren{n\log_2k - H(\mathbf{z}|\hat{\mathbf{z}}, \mathbf{A}=1)}\cdot \bbP[\mathbf{A}=1]\\
    & \leq 1 + n\log_2k - \tau\cdot\Paren{n\log_2k - H(\mathbf{z}|\hat{\mathbf{z}}, \mathbf{A}=1)}\,,
\end{align*}
where the last inequality follows from the fact that $\bbP[\mathbf{A}=1]\geq \tau$ and the fact that $n\log_2 k - H(\mathbf{z}|\hat{\mathbf{z}}, \mathbf{A}=1)\geq 0$ (because $\mathbf{z}\in[k]^n$). Now since $\mathbf{z}$ is a uniform random variable in $[k]^n$, we have
$H(\mathbf{z})=\log_2\Paren{k^n}=n\log_2 k\,,$
and hence
\begin{equation}
    \label{eq:mutual_information_lower_bound}
\begin{aligned}
    I(\mathbf{z};\hat{\mathbf{z}}) &= H(\mathbf{z}) - H(\mathbf{z}|\hat{\mathbf{z}})\\
    & \geq \tau\cdot\Paren{n\log_2k - H(\mathbf{z}|\hat{\mathbf{z}}, \mathbf{A}=1)} - 1\,.
\end{aligned}
\end{equation}

Now we will focus on upper bounding $H(\mathbf{z}|\hat{\mathbf{z}}, \mathbf{A}=1)$. For every $\hat{z}\in [k]^n$, we have
\begin{equation}
    \label{eq:conditional_entropy_z}
H(\mathbf{z}|\hat{\mathbf{z}}=\hat{z}, \mathbf{A}=1) \leq \log_2|Z(\hat{z},\rho)|\,,
\end{equation}
where
\begin{align*}
    Z(\hat{z},\rho) &= \Set{z\in[k]^n: \max_{\pi \in P_k}\sum_i\frac{1}{n}
    \bracbb{z_i=\pi(\hat{z}_{i})}\geq \frac{1}{k}+\rho}=\bigcup_{\pi\in P_k} Z(\hat{z},\rho, \pi)\,,
\end{align*}
and
\begin{align*}
    Z(\hat{z},\rho,\pi) &= \Set{z\in[k]^n: \sum_i\frac{1}{n}
    \bracbb{z_i=\pi(\hat{z}_{i})}\geq \frac{1}{k}+\rho}\,.
\end{align*}
Hence,
$$|Z(\hat{z},\rho)|\leq \sum_{\pi\in P_k}|Z(\hat{z},\rho,\pi)|\,.$$

We will further divide $Z(\hat{z},\rho,\pi)$ as follows:
$$Z(\hat{z},\rho,\pi) = \bigcup_{\beta n\leq m\leq n} Z(\hat{z},\rho,\pi, m)\,,$$
where
$$\beta=\frac{1}{k}+\rho\,,$$
and
\begin{align*}
    Z(\hat{z},\rho,\pi,m) &= \Set{z\in[k]^n: \sum_i
    \bracbb{z_i=\pi(\hat{z}_{i})}=m}\,.
\end{align*}
It is not hard to see that
$$|Z(\hat{z},\rho,\pi,m)| = \binom{n}{m} \cdot(k-1)^{n-m}\,.$$

By defining $\beta_m=\frac{m}{n}$ and using Stirling's formula\footnote{For large $n$, we have $\log_2(n!)=n\log_2 n - n\log_2 e\pm O(\log_2 n)$.}, we get:
\begin{align*}
    \log_2 |Z(\hat{z},\rho,\pi,m)| &= n\log_2 n - n\log_2 e \\
    &\quad \quad - m \log_2 m + m\log_2 e -  (n-m) \log_2(n-m) + (n-m)\log_2 e \\
    &\quad \quad\pm O(\log n) + (n-m)\log_2(k-1)\\
    &= n\log_2 n - \beta_m n\log_2 (\beta_m n) - (1-\beta_m) n\log_2 ((1-\beta_m) n) \\
    &\quad\quad+ (1-\beta_m)n\log_2(k-1)\pm O(\log n)\\
    &= \Paren{h_2(\beta_m) + (1-\beta_m)\cdot\log_2(k-1)}\cdot n\pm O(\log n)\,.
\end{align*}

By taking derivatives and analyzing the function $g(\beta_m) = h_2(\beta_m) + (1-\beta_m)\cdot\log_2(k-1)$, we can show that $g$ is decreasing after $\beta_m\geq\frac{1}{k}$, and hence for $\beta_m\geq \beta=\frac{1}k+\rho\geq \frac1k$, we have $g(\beta_m)\leq g(\beta)$. In particular, for every $m$ satisfying $\beta n\leq m \leq n$, we have
\begin{align*}
    \log_2 |Z(\hat{z},\rho,\pi,m)| \leq \Paren{h_2(\beta) + (1-\beta)\cdot\log_2(k-1)}\cdot n+ O(\log n)\,.
\end{align*}
Therefore,
\begin{align*}
    |Z(\hat{z},\rho)|&\leq \sum_{\pi\in P_k}\sum_{\beta n\leq m\leq n} |Z(\hat{z},\rho,\pi, m)|\\
    &\leq \sum_{\pi\in P_k}\sum_{\beta n\leq m\leq n} 2^{\Paren{h_2(\beta) + (1-\beta)\cdot\log_2(k-1)}\cdot n+ O(\log n)}\\
    &= k!\cdot n \cdot 2^{\Paren{h_2(\beta) + (1-\beta)\cdot\log_2(k-1)}\cdot n+ O(\log n)}\,,
\end{align*}
and hence
\begin{align*}
    \log_2|Z(\hat{z},\rho)|&\leq \Paren{h_2(\beta) + (1-\beta)\cdot\log_2(k-1)}\cdot n+ O(k\log k + \log n)\,.
\end{align*}

Since this is true for every $\hat{z}\in [k]^n$, we get from \eqref{eq:conditional_entropy_z} that

$$
H(\mathbf{z}|\hat{\mathbf{z}}, \mathbf{A}=1) \leq \Paren{h_2(\beta) + (1-\beta)\cdot\log_2(k-1)}\cdot n+ O(k\log k + \log n)\,.
$$

Combining this with \eqref{eq:mutual_information_lower_bound}, we get
\begin{equation}
    \label{eq:eq:mutual_information_lower_bound_in_n}
\begin{aligned}
    I(\mathbf{z};\hat{\mathbf{z}}) &\geq \tau\Paren{n\log_2 k - \Paren{h_2(\beta) + (1-\beta)\cdot\log_2(k-1)}n-O(\log n + k\log k)} - 1 \\
    &\geq \frac{\tau\cdot n}{2}\Paren{\log_2 k - h_2(\beta) - (1-\beta)\cdot\log_2(k-1)} \,,
\end{aligned}
\end{equation}
where the last inequality assumes\footnote{It is worth noting that if $\beta>1/k\,,$ then $\log_2 k - h_2(\beta) - (1-\beta)\cdot\log_2(k-1)>0\,,$ as we will show in the next subsection.} that $n$ is large enough (and in particular $n\gg k\log k$).

\subsection{Putting everything together}\label{section:putting-everything-together}
\begin{proof}[Proof of \cref{theorem:lowerbound-formal}]

Assume that there is a $(\rho,\tau,\alpha,t)$-weak-recovery blackbox estimator $\hat{z}$ for $\mathbf{z}$ and assume that $n$ is large enough. Let $(\hat{\mathbf{X}}_{\ell})_{\ell\in[t]}$ be $\alpha$-estimates of $(\mathbf{x}_{\ell})_{\ell\in[t]}$ satisfying \eqref{eq:mutual_information_upper_bound}, i.e.,
$$
I(\mathbf{z};\hat{\mathbf{X}}_1,\ldots,\hat{\mathbf{X}}_{t})\leq C\cdot\overline{\alpha}\cdot t \cdot n\,.    
$$

Let $\hat{\mathbf{z}}=\hat{z}(\hat{\mathbf{X}}_1,\ldots,\hat{\mathbf{X}}_\ell)$. From the data-processing inequality, we have
$$
I(\mathbf{z};\hat{\mathbf{z}})\leq I(\mathbf{z};\hat{\mathbf{X}}_1,\ldots,\hat{\mathbf{X}}_{t})\leq C\cdot\overline{\alpha}\cdot t \cdot n\,.    
$$

On the other hand, since $\hat{z}$ is a $(\rho,\tau,\alpha,t)$-weak-recovery blackbox estimator and since $(\hat{\mathbf{x}}_{\ell})_{\ell\in[t]}$ are $\alpha$-estimates of $(\mathbf{x}_{\ell})_{\ell\in[t]}$, it follows that $\hat{\mathbf{z}}$ satisfies \eqref{eq:heterogeneous-weak-recovery-rho} with probability $1-\tau$. It follows from \eqref{eq:eq:mutual_information_lower_bound_in_n} that for $n$ large enough, we have
$$
    I(\mathbf{z};\hat{\mathbf{z}}) \geq \frac{\tau\cdot n}{2}\cdot\Paren{\log_2 k - h_2(\beta) - (1-\beta)\cdot\log_2(k-1)} \,.
$$

We conclude that
$$\frac{\tau\cdot n}{2}\cdot\Paren{\log_2 k - h_2(\beta) - (1-\beta)\cdot\log_2(k-1)}\cdot n \leq C\cdot\overline{\alpha}\cdot t \cdot n\,.$$
Therefore, we must have
$$t\geq \tau\cdot\frac{\log_2 k - h_2(\beta) - (1-\beta)\cdot\log_2(k-1)}{2 C\cdot \overline{\alpha}}\,.$$

Now define
\begin{align*}
    l(\beta) &= \log_2 k - h_2(\beta) - (1-\beta)\cdot\log_2(k-1)\\
    &=\log_2 k +\beta\log_2\beta + (1-\beta)\log_2(1-\beta) - (1-\beta)\cdot\log_2(k-1) \,,
\end{align*}
so that
$$t\geq \frac{\tau\cdot l(\beta)}{2 C\cdot \overline{\alpha}}=\frac{\tau\cdot l(1/k+\rho)}{2 C\cdot \overline{\alpha}}\,.$$

Let us analyze the function $l(\beta)$:

\begin{itemize}
    \item A quick calculation shows that
$$l(1/k)=0\,.$$
\item The derivative of $l$ is
\begin{align*}
    l'(\beta)&=\log_2\beta + \frac{1}{\ln 2} - \log_2(1-\beta) - \frac{1}{\ln 2} + \log_2(k-1)\\
&=\log_2\beta - \log_2(1-\beta) + \log_2(k-1) \,,
\end{align*}
and hence
$$l'(1/k)=0\,.$$
\item The second derivative of $l$ is
\begin{align*}
    l''(\beta)&=\frac{1}{\beta\ln 2}+\frac{1}{(1-\beta)\ln 2}> 0\,,\quad \forall\beta\in(0,1) \,.
\end{align*}
\end{itemize}
Hence, $l'(\beta)>0$ for $\beta\in(1/k,1)\,,$ and since $l(1/k)=0$ we can see that $l(1/k+\rho)>0$ whenever $\rho>0\,.$ Furthermore, a quick calculation reveals that for fixed $\rho$ we have\footnote{Note that $\beta\log_2\beta + (1-\beta)\log_2(1-\beta)=-h_2(\beta)$ and since $0\leq h_2(\beta)\leq 1\,,$ we can see that $\lim_{k\to\infty} \frac{h_2(\beta)}{\log_2(k)}=0\,.$ Hence $\lim_{k\to\infty}\frac{l(1/k+\rho)}{\log_2(k)}$ can be simplified as $\lim_{k\to\infty} 1-(1-\rho-\frac{1}{k})\frac{\log_2(k-1)}{\log_2(k)}=\rho\,.$}
$$\lim_{k\to\infty}\frac{l(1/k+\rho)}{\log_2(k)}=\rho>0\,,$$
which means that
$$\min_{k\geq 2}\frac{l(1/k+\rho)}{\log_2(k)}>0\,,$$
and hence $l(1/k+\rho)=\Omega_\rho(\log_2(k))\,.$ We conclude that
$$t\geq \Omega_\rho\Paren{ \frac{\tau\cdot \log k}{\overline{\alpha}}}\,.$$
\end{proof}
\section{Deferred proofs}\label{section:deferred-proofs}
We present here proofs deferred in the main body of the paper.

\paragraph{Deferred proofs of \cref{section:specialized-weak-recovery}}.

To obtain \cref{theorem:weak-recovery-sbm} we need to introduce results about robust weak recovery.

\begin{definition}[$\mu$-node corrupted, balanced $2$ communities SBM]\label{definition:node-corrupted}
    Let $\mu \in [0,1]\,.$
    Let $x\in \set{\pm 1}^n$ be a vector satisfying $\sum_i x_i=0$ and let $\mathbf{G}^0\sim \sbmx{x}\,.$ An adversary may choose up to $\mu\cdot n$ vertices in $\mathbf{G}^0$ and arbitrarily
    modify edges (and non-edges) incident to at least one of them to produce the corrupted graph $G$. We write $G\stackrel{\mu}{\approx} \mathbf{G}^0$ to denote that $G$ is a $\mu$-node corrupted version of $\mathbf{G}^0$.
\end{definition}

In the context of node corrupted graphs, the definition of weak recovery is still with respect to the original vector $x$ as defined in \cref{eq:k-weak-recovery}. It is known that node robust weak recovery is achievable.

\begin{theorem}[Implicit in \cite{ding2023node}]\label{theorem:node-robust-weak-recovery}
Let $n,d,\eps >0$ be satisfying $d\cdot \eps^2-1=:\delta>0\,.$ There exist:
\begin{itemize}
    \item constants $0<\mu_\delta<1$ and $0<C_\delta<1\,,$ and
    \item a (randomized) polynomial time algorithm\footnote{The subscript $\mathrm{r}$ in $\hat{\mathbf{X}}_{\mathrm{r}}$ stands for ``robust".} $\hat{\mathbf{X}}_{\mathrm{r}}$ taking a graph $G$ of $n$ vertices as input and producing a matrix $\hat{\mathbf{X}}_{\mathrm{r}}(G)\in \Brac{-1, +1}^{n\times n}$ as output,
\end{itemize}
such that $\hat{\mathbf{X}}_{\mathrm{r}}$ is a successful weak-recovery recovery algorithm robust against any $\mu$-node corruption for all $\mu\leq \mu_\delta$. More formally, for every $x\in \set{\pm 1}^n$ satisfying $\sum_i x_i=0\,,$ and every $\mu\leq \mu_\delta\,,$ we have
\begin{align*}
    \E_{\mathbf{G}^0\sim \sbmx{x}}\Brac{ \min_{G:G\stackrel{\mu}{\approx} \mathbf{G}^0} \iprod{\hat{\mathbf{X}}_{\mathrm{r}}(G), \dyad{x}}}\geq C_\delta\cdot n^2\,.
\end{align*}
\end{theorem}

We can use \cref{theorem:node-robust-weak-recovery} to obtain \cref{theorem:weak-recovery-sbm}.

Let $\gamma=\Abs{p-\frac{1}{2}}$ be the unbalancedness in the vector of labels of $\mathbf{x}$, where $p=\bbP(\mathbf{x}_i=+1)$.

The main idea behind the algorithm in \cref{theorem:weak-recovery-sbm} is to first distinguish whether the unbalancedness $\gamma$ is sufficiently small or not. If it is sufficiently small, then we apply the robust algorithm of \cref{theorem:node-robust-weak-recovery}. Otherwise, we can achieve weak-recovery by relying on the degree of a vertex to estimate its community label.

In the following two lemmas, we treat the case where the unbalancedness $\gamma$ is sufficiently small:

\begin{lemma}
\label{lemma:robust-weak-recovery-on-almost-balanced-communities} Let $n,d,\eps,p,\mathbf{x}=(\mathbf{x}_i)_{i\in[n]}$ and $\mathbf{G}\sim\sbm(\mathbf{x})$ be as in \cref{theorem:weak-recovery-sbm} and let $\delta=\frac{\eps^2 d}{4}-1>0$. Let $\mu_\delta,C_\delta$ and $\hat{\mathbf{X}}_{\mathrm{r}}$ be\footnote{It is worth noting that since \cref{theorem:node-robust-weak-recovery} assumes that $\sum_i x_i=0\,,$ then $n$ must be even in \cref{theorem:node-robust-weak-recovery}. However, in \cref{lemma:robust-weak-recovery-on-almost-balanced-communities} we would like $n$ to be general. Hence, if $n$ is odd, we apply the following procedure: (1) we add a fictitious vertex $n+1$ which is not incident to any vertex in $[n]$ and we call the resulting graph (having $[n+1]$ as its set of vertices) as $\tilde{\mathbf{G}}\,,$ (2) we apply $\hat{\mathbf{X}}_{\mathrm{r}}$ on $\tilde{\mathbf{G}}\,,$ and (3) we take the submatrix of $\hat{\mathbf{X}}_{\mathrm{r}}(\tilde{\mathbf{G}})$ induced by the vertices in $[n]$. We still denote the overall algorithm as $\hat{\mathbf{X}}_{\mathrm{r}}$.} as in \cref{theorem:node-robust-weak-recovery} and define
$$\mu'_\delta=\frac{1}{100}\min\set{\mu_\delta, C_\delta}\,.$$
If the unbalancedness $\gamma=\Abs{\frac{1}{2}-p}$ of $\mathbf{x}$ satisfies $\gamma\leq \mu'_\delta\,,$ then for $n$ large enough, we have
\begin{align*}
    \E\Brac{ \iprod{\hat{\mathbf{X}}_{\mathrm{r}}(\mathbf{G}), \dyad{\mathbf{x}}}}\geq \frac34
    C_\delta\cdot n^2\,.
\end{align*}
\end{lemma}
\begin{proof}
For the sake of simplicity, we will only treat the case where $n$ is even. Since $\hat{\mathbf{X}}_{\mathrm{r}}$ is robust against node corruptions, it is not hard to see that the proofs can be adapted to the case where $n$ is odd.

For every $x\in \set{\pm}^{n}$, let $n_+(x)=|\set{i\in [n]:x_i=+1}|$ and $n_-(x)=|\set{i\in [n]:x_i=-1}|$. Since $\bbP(\mathbf{x}_i=+1)=p$, then by the law of large numbers we know that $n_+(\mathbf{x})$ and $n_-(\mathbf{x})$ concentrate around $pn$ and $(1-p)n$, respectively. Furthermore, since $\gamma=\Abs{\frac{1}{2}-p}\leq\mu'_\delta$, we can use standard concentration inequalities to show that with probability at least $1-O(n^{-10})$, the random vector $\mathbf{x}$ satisfies the event
$$\mathcal{E}=\Set{x\in \set{\pm}^{n}: \Abs{\frac{n_+(x)}{n}-\frac{1}{2}}\leq 2\mu'_\delta\text{ and }\Abs{\frac{n_-(x)}{n}-\frac{1}{2}}\leq 2\mu'_\delta}\,.$$

Now since $\bbP(\mathbf{x}\in \mathcal{E})=1-O(n^{-10})$ and since $\Abs{\iprod{\hat{\mathbf{X}}_{\mathrm{r}}(\mathbf{G}), \dyad{\mathbf{x}}}}\leq n^2$, it is not hard to see that
\begin{equation}
\label{eq:inner-product-expectation-condition-on-typical-event}
    \E\Brac{ \iprod{\hat{\mathbf{X}}_{\mathrm{r}}(\mathbf{G}), \dyad{\mathbf{x}}}}=\E\Brac{ \iprod{\hat{\mathbf{X}}_{\mathrm{r}}(\mathbf{G}), \dyad{\mathbf{x}}}\given \mathbf{x}\in \mathcal{E}} \pm o(1) \,,
\end{equation}
so we can focus on studying $\E\Brac{ \iprod{\hat{\mathbf{X}}_{\mathrm{r}}(\mathbf{G}), \dyad{\mathbf{x}}}\given \mathbf{x}\in \mathcal{E}}$.

Now fix $x\in \mathcal{E}$ and condition on the event that $\mathbf{x}=x$. From the definition of $\mathcal{E}$ it is not hard to see that there is $x'\in\set{\pm}^n$ satisfying $\sum_{i\in[n]}x_i'=0$ and $$|D_{x,x'}|\leq 4\mu_\delta'n\,,$$
where
$$D_{x,x'}=|i\in[n]:x_i\neq x_i'|\,.$$

Now construct a random graph $\mathbf{G}'$ as follows:
\begin{itemize}
    \item If $i,j\in [n]\setminus D_{x,x'}\,,$ i.e., if $x_i=x_i'$ and $x_j=x_j'\,,$ then we let $\{i,j\}\in\mathbf{G}'$ if and only if $\{i,j\}\in\mathbf{G}$.
    \item If either $i\in D_{x,x'}$ or $j\in D_{x,x'}\,$ then we put the edge $\{i,j\}$ in $\mathbf{G}'$ with probability $\frac{d}{n}\Paren{1+\frac{\eps}{2}x_i'\cdot x_j'}\,.$
    \item The events $\Paren{\{i,j\}\in \mathbf{G}'}_{i,j\in[n]}$ are mutually independent.
\end{itemize}

It is not hard to see that:
\begin{itemize}
    \item $\mathbf{G}'\sim\sbm(x')\,,$ and
    \item $\mathbf{G}\stackrel{4\mu_\delta'}{\approx}\mathbf{G}'\,,$ i.e., $\mathbf{G}$ can be obtained from $\mathbf{G}'$ by adding or removing edges incident to at most $4\mu_\delta'n$ vertices.
\end{itemize}
Since $4\mu'_\delta\leq \frac{4}{100}\mu_\delta\leq \mu_\delta$, it follows from \cref{theorem:node-robust-weak-recovery} that
\begin{equation}
\label{eq:node-robust-weak-recovery-xprime}
    \E\Brac{ \iprod{\hat{\mathbf{X}}_{\mathrm{r}}(\mathbf{G}), \dyad{x'}}\given \mathbf{x}=x}\geq 
    C_\delta\cdot n^2\,.
\end{equation}

Now notice that
\begin{align*}
    \iprod{\hat{\mathbf{X}}_{\mathrm{r}}(\mathbf{G}), \dyad{x}} = \iprod{\hat{\mathbf{X}}_{\mathrm{r}}(\mathbf{G}), \dyad{x'}} + \iprod{\hat{\mathbf{X}}_{\mathrm{r}}(\mathbf{G}), \dyad{x}-\dyad{x'}}\,,
\end{align*}
and
\begin{align*}
    \Abs{\iprod{\hat{\mathbf{X}}_{\mathrm{r}}(\mathbf{G}), \dyad{x}-\dyad{x'}}}&\leq \Norm{\hat{\mathbf{X}}_{\mathrm{r}}(\mathbf{G})}_\infty\cdot\Norm{\dyad{x}-\dyad{x'}}_1\leq 1\cdot \Norm{\dyad{x}-\dyad{x'}}_1\\
    &\leq 2\Norm{x-x'}_1\cdot n = 4|D_{x,x'}|\cdot n\leq 16\mu_\delta'\cdot n^2\,.
\end{align*}
Combining this with \eqref{eq:node-robust-weak-recovery-xprime}, we get
\begin{align*}
    \E\Brac{ \iprod{\hat{\mathbf{X}}_{\mathrm{r}}(\mathbf{G}), \dyad{x}}\given \mathbf{x}=x}\geq 
    \Paren{C_\delta -16\mu_\delta'}\cdot n^2\geq \frac{84}{100}C_\delta\cdot n^2\,.
\end{align*}
Now since this is true for all $x\in\mathcal{E}\,,$ we conclude that
\begin{align*}
    \E\Brac{ \iprod{\hat{\mathbf{X}}_{\mathrm{r}}(\mathbf{G}), \dyad{\mathbf{x}}}\given \mathbf{x}\in\mathcal{E}}\geq 
    \Paren{C_\delta -16\mu_\delta'}\cdot n^2\geq \frac{84}{100}C_\delta\cdot n^2\,,
\end{align*}
where the last inequality is true because $\mu_\delta'=\frac{1}{100}\min\set{\mu_\delta, C_\delta}$.
Combining the above with \eqref{eq:inner-product-expectation-condition-on-typical-event}, we can deduce that for $n$ large enough, we have
\begin{align*}
    \E\Brac{ \iprod{\hat{\mathbf{X}}_{\mathrm{r}}(\mathbf{G}), \dyad{\mathbf{x}}}}\geq \frac34
    C_\delta\cdot n^2\,.
\end{align*}
\end{proof}

The following lemma takes the algorithm of \cref{lemma:robust-weak-recovery-on-almost-balanced-communities} and applies a symmetrization argument in order to get ``a positive correlation at the edge level".

\begin{lemma}[Pair-wise weak recovery for sufficiently balanced 2 communities stochastic block mode]
\label{lemma:weak-recovery-sufficiently-balanced} Let $n,d,\eps,p,\mathbf{x}=(\mathbf{x}_i)_{i\in[n]}$ and $\mathbf{G}\sim\sbm(\mathbf{x})$ be as in \cref{theorem:weak-recovery-sbm}. Let $\delta=\frac{\eps^2 d}{4}-1>0$ and let $\gamma=\Abs{\frac{1}{2}-p}$ be the unbalancedness of $\mathbf{x}$. There exist constants $\mu_\delta'>0$ and $C_\delta'>0$ and a randomized polynomial-time algorithm\footnote{The subscript $\mathrm{sb}$ in $\hat{\mathbf{X}}_{\mathrm{sb}}$ stands for ``sufficiently balanced".} $\hat{\mathbf{X}}_{\mathrm{sb}}$ taking $\mathbf{G}$ as input and producing a matrix $\hat{\mathbf{X}}_{\mathrm{sb}}(\mathbf{G})\in[-1,+1]^{n\times n}$ such that if $$\gamma\leq \mu_\delta'\,,$$ then for every $i,j\in[n]$ with $i\neq j$, we have
\begin{align*}
    C_{\delta}'\leq \E\Brac{\hat{\mathbf X}_{\mathrm{sb}}(\mathbf{G})_{\ij}\suchthat \mathbf{x}_i=\mathbf{x}_j} - \E\Brac{\hat{\mathbf X}_{\mathrm{sb}}(\mathbf{G})_{\ij}\suchthat \mathbf{x}_i\neq \mathbf{x}_j}\,.
\end{align*}
\end{lemma}
\begin{proof}
Let $\mu_\delta,C_\delta$ and $\hat{\mathbf{X}}_{\mathrm{r}}$ be as in \cref{theorem:node-robust-weak-recovery}, and let $\mu_\delta'=C_\delta'=\frac{1}{100}\min\set{\mu_\delta, C_\delta}$. \cref{lemma:robust-weak-recovery-on-almost-balanced-communities} shows that
\begin{align*}
    \E\Brac{ \iprod{\hat{\mathbf{X}}_{\mathrm{r}}(\mathbf{G}), \dyad{\mathbf{x}}}}\geq \frac34
    C_\delta\cdot n^2\,.
\end{align*}

The algorithm $\hat{\mathbf{X}}_{\mathrm{sb}}(\mathbf{G})$ can be obtained by ``symmetrizing" the algorithm $\hat{\mathbf{X}}_{\mathrm{r}}$ as follows: Let ${\bm \sigma}$ be a (uniformly) random permutation $[n]\to[n]$ and let $\mathbf{G}_{{\bm \sigma}}$ be the graph obtained from $\mathbf{G}$ by ${\bm \sigma}$-permuting its vertices, i.e., we let the edge\footnote{For simplicity, We denote ${\bm \sigma}(i)$ as ${\bm \sigma}_i$.} $\{{\bm \sigma}_i,{\bm \sigma}_j\}$ belong to $\mathbf{G}_{\bm \sigma}$ if and only if $\{i,j\}\in \mathbf{G}$. We define the matrix $\hat{\mathbf{X}}_{\mathrm{sb}}(\mathbf{G})\in[-1,+1]^n$ as follows:
$$\hat{\mathbf{X}}_{\mathrm{sb}}(\mathbf{G})_{ij}=\hat{\mathbf{X}}_{r}(\mathbf{G}_{{\bm \sigma}})_{{\bm \sigma}_i{\bm \sigma}_j}\,.$$
In other words, we apply the random permutation ${\bm \sigma}$ to graph $\mathbf{G}\,,$ we apply the algorithm $\hat{\mathbf{X}}_{\mathrm{r}}\,,$ and then we apply the inverse of the permutation on the resulting matrix.

Due to the symmetry of the SBM distribution, it is not hard to see that for every $i,j\in[n]$ with $i\neq j$, we have
\begin{align*}
    \E\Brac{\hat{\mathbf X}_{\mathrm{sb}}(\mathbf{G})_{\ij}\cdot \mathbf{x}_i\mathbf{x}_j}&=\sum_{\substack{i',j'\in[n]:\\i'\neq j'}} \E\Brac{\hat{\mathbf X}_{\mathrm{sb}}(\mathbf{G})_{\ij}\cdot \mathbf{x}_i\mathbf{x}_j\given {\bm \sigma}_i=i',{\bm \sigma}_j=j'}\cdot\bbP\Paren{{\bm \sigma}_i=i',{\bm \sigma}_j=j'}\\
    &=\sum_{\substack{i',j'\in[n]:\\i'\neq j'}} \E\Brac{\hat{\mathbf{X}}_{r}(\mathbf{G}_{{\bm \sigma}})_{{\bm \sigma}_i{\bm \sigma}_j}\cdot \mathbf{x}_i\mathbf{x}_j\given {\bm \sigma}_i=i',{\bm \sigma}_j=j'}\cdot\frac{1}{n(n-1)}\\
    &\stackrel{(\ast)}{=}\frac{1}{n(n-1)}\sum_{\substack{i',j'\in[n]:\\i'\neq j'}} \E\Brac{\hat{\mathbf{X}}_{r}(\mathbf{G})_{{\bm \sigma}_i{\bm \sigma}_j}\cdot \mathbf{x}_{{\bm \sigma}_i}\mathbf{x}_{{\bm \sigma}_j}\given {\bm \sigma}_i=i',{\bm \sigma}_j=j'}\\
    &=\frac{1}{n(n-1)}\sum_{\substack{i',j'\in[n]:\\i'\neq j'}}\E\Brac{\hat{\mathbf X}_{\mathrm{r}}(\mathbf{G})_{i'j'}\cdot \mathbf{x}_{i'}\mathbf{x}_{j'}}\\
    &=\frac{1}{n(n-1)}\E\Brac{ \iprod{\hat{\mathbf{X}}_{\mathrm{r}}(\mathbf{G}), \dyad{\mathbf{x}}}} - \frac{1}{n(n-1)}\sum_{i\in[n]}\E\Brac{\hat{\mathbf X}_{\mathrm{r}}(\mathbf{G})_{ii}\cdot \mathbf{x}_i^2}\\
    &\geq \frac{1}{n(n-1)}\cdot \frac34
    C_\delta\cdot n^2 -\frac{1}{n-1}\\
    &\geq \frac34C_\delta - o(1)\,,
\end{align*}
where $(\ast)$ follows from the symmetry of the SBM distribution under the simultaneous permutation of vertices and labels.

Now notice that
\begin{align*}
    \E\Brac{\hat{\mathbf X}_{\mathrm{sb}}(\mathbf{G})_{\ij}\cdot \mathbf{x}_i\mathbf{x}_j}&= \E\Brac{\hat{\mathbf X}_{\mathrm{sb}}(\mathbf{G})_{\ij}\suchthat \mathbf{x}_i=\mathbf{x}_j}\cdot\bbP\Paren{\mathbf{x}_i=\mathbf{x}_j} - \E\Brac{\hat{\mathbf X}_{\mathrm{sb}}(\mathbf{G})_{\ij}\suchthat \mathbf{x}_i\neq \mathbf{x}_j}\cdot\bbP\Paren{\mathbf{x}_i\neq \mathbf{x}_j}\\
    &=\E\Brac{\hat{\mathbf X}_{\mathrm{sb}}(\mathbf{G})_{\ij}\suchthat \mathbf{x}_i=\mathbf{x}_j}\cdot\Paren{p^2+(1-p)^2} - \E\Brac{\hat{\mathbf X}_{\mathrm{sb}}(\mathbf{G})_{\ij}\suchthat \mathbf{x}_i\neq \mathbf{x}_j}\cdot\Paren{2p(1-p)}\\
    &=\E\Brac{\hat{\mathbf X}_{\mathrm{sb}}(\mathbf{G})_{\ij}\suchthat \mathbf{x}_i=\mathbf{x}_j}\cdot\Paren{\frac{1}{2}+2\gamma^2} - \E\Brac{\hat{\mathbf X}_{\mathrm{sb}}(\mathbf{G})_{\ij}\suchthat \mathbf{x}_i\neq \mathbf{x}_j}\Paren{\frac{1}{2}-2\gamma^2}\\
    &\leq \frac{1}{2}\cdot\E\Brac{\hat{\mathbf X}_{\mathrm{sb}}(\mathbf{G})_{\ij}\suchthat \mathbf{x}_i=\mathbf{x}_j}+2\gamma^2 - \frac{1}{2}\cdot\E\Brac{\hat{\mathbf X}_{\mathrm{sb}}(\mathbf{G})_{\ij}\suchthat \mathbf{x}_i\neq \mathbf{x}_j}+2\gamma^2\,,
\end{align*}
where in the last inequality we used the fact that $\Abs{\hat{\mathbf X}_{\mathrm{sb}}(\mathbf{G})_{\ij}}\leq 1$. We can deduce that for $n$ large enough, we have
\begin{align*}
    \E\Brac{\hat{\mathbf X}_{\mathrm{sb}}(\mathbf{G})_{\ij}\suchthat \mathbf{x}_i=\mathbf{x}_j}&- \E\Brac{\hat{\mathbf X}_{\mathrm{sb}}(\mathbf{G})_{\ij}\suchthat \mathbf{x}_i\neq \mathbf{x}_j}\\
    &\geq 2\E\Brac{\hat{\mathbf X}_{\mathrm{sb}}(\mathbf{G})_{\ij}\cdot \mathbf{x}_i\mathbf{x}_j} - 8\gamma^2\geq \frac{3}{2}C_\delta - o(1) - 8\gamma^2\geq C_\delta\,,
\end{align*}
where the last inequality follows from the fact that
$$8\gamma^2\leq 8\mu_\delta'^2\leq 8\Paren{\frac{1}{100}C_\delta}^2\leq \frac{C_\delta}{100}\,.$$

By picking $C_\delta'=C_\delta\,,$ the lemma follows.
\end{proof}

Now we turn to show that if $\mathbf{x}$ is sufficiently unbalanced, then there exists an efficient algorithm that achieves pair-wise weak recovery.

\begin{lemma}[Pair-wise weak recovery for sufficiently unbalanced 2 communities stochastic block mode]
\label{lemma:weak-recovery-sufficiently-unbalanced} Let $n,d,\eps,p,\mathbf{x}=(\mathbf{x}_i)_{i\in[n]}$ and $\mathbf{G}\sim\sbm(\mathbf{x})$ be as in \cref{theorem:weak-recovery-sbm}. Further assume
that $p\in[0,1]$. Let $\delta=\frac{\eps^2 d}{4}-1>0$ and let $\gamma=\Abs{\frac{1}{2}-p}$ be the unbalancedness of $\mathbf{x}$, and let $\mu_\delta'$ be as in \cref{lemma:weak-recovery-sufficiently-balanced}. There exists a constant $C_{\delta}''>0$ and a randomized polynomial-time algorithm
\footnote{The subscript $\mathrm{su}$ in $\hat{\mathbf{X}}_{\mathrm{su}}$ stands for ``sufficiently unbalanced".}
$\hat{\mathbf{X}}_{\mathrm{su}}$ taking $\mathbf{G}$ as input and producing a matrix $\hat{\mathbf{X}}_{\mathrm{su}}(\mathbf{G})\in[-1,+1]^{n\times n}$ such that if $$\gamma\geq \frac12\mu_\delta'\,,$$ then for every $i,j\in[n]$ with $i\neq j$, we have
\begin{align*}
    C_{\delta}''\leq \E\Brac{\hat{\mathbf X}_{\mathrm{sb}}(\mathbf{G})_{\ij}\suchthat \mathbf{x}_i=\mathbf{x}_j} - \E\Brac{\hat{\mathbf X}_{\mathrm{sb}}(\mathbf{G})_{\ij}\suchthat \mathbf{x}_i\neq \mathbf{x}_j}\,.
\end{align*}
\end{lemma}
\begin{proof}
For the sake of simplicity we may assume without loss of generality that $p>\frac{1}{2}$\,, i.e., $p=\frac12+\gamma$ and hence ``$+1$" is the larger community in expectation.

For every $i,j\in[n]$, let
$$\mathrm{\textbf{deg}}_{\neq j}(i)=\Abs{\Set{v\in[n]\setminus \{i,j\}:\{i,v\}\in\mathbf{G}}}$$ be the number of vertices in $[n]\setminus\set{i,j}$ which are adjacent to $i$ in $\mathbf{G}$.

Let $C>0$ be a large enough constant (to be chosen later) and let
\begin{align*}
    \hat{x}_i^{(j)}(\mathbf{G})&=\frac{1}{C}\Paren{\mathrm{\textbf{deg}}_{\neq j}(i)-d(1-2/n)}\cdot\Bracbb{\Abs{\mathrm{\textbf{deg}}_{\neq j}(i)-d(1-2/n)}\leq C}\in[-1,1]\,,
\end{align*}
and define the matrix $\hat{\mathbf{X}}_{\mathrm{su}}(\mathbf{G})\in[-1,+1]^{n\times n}$ as:
$$
\hat{\mathbf{X}}_{\mathrm{su}}(\mathbf{G})_{ij}= \hat{x}_i^{(j)}(\mathbf{G})\cdot \hat{x}_j^{(i)}(\mathbf{G})\,.
$$

It is not hard to see that given $(\mathbf{x}_i,\mathbf{x}_j)\,,$ the random variables $\mathrm{\textbf{deg}}_{\neq j}(i)$ and $\mathrm{\textbf{deg}}_{\neq i}(j)$ are conditionally independent. Therefore, $\hat{x}_i^{(j)}(\mathbf{G})$ and $ \hat{x}_j^{(i)}(\mathbf{G})$ are conditionally independent given $(\mathbf{x}_i,\mathbf{x}_j)\,,$ hence
\begin{align*}
    \E\Brac{\hat{\mathbf{X}}_{\mathrm{su}}(\mathbf{G})_{ij}\given \mathbf{x}_i,\mathbf{x}_j}
    &= \E\Brac{\hat{x}_i^{(j)}(\mathbf{G})\given \mathbf{x}_i,\mathbf{x}_j}\cdot \E\Brac{\hat{x}_j^{(i)}(\mathbf{G})\given \mathbf{x}_i,\mathbf{x}_j}\\
    &= \E\Brac{\hat{x}_i^{(j)}(\mathbf{G})\given \mathbf{x}_i}\cdot \E\Brac{\hat{x}_j^{(i)}(\mathbf{G})\given \mathbf{x}_j}\,.
\end{align*}

Now, for every $v\in[v]\setminus\set{i,j}\,,$ we have
\begin{align*}
    \bbP\Paren{\{i,v\}\in\mathbf{G}\given \mathbf{x}_i}
    &=\E\Brac{ \bbP\Paren{\{i,v\}\in\mathbf{G}\given \mathbf{x}_i,\mathbf{x}_v}\given \mathbf{x}_i}=\E\Brac{ \frac{d}{n}\Paren{1+\frac{1}{2}\eps\cdot\mathbf{x}_i\mathbf{x}_v}\given \mathbf{x}_i}\\
    &=\frac{d}{n}\Paren{1+\frac{1}{2}\eps\cdot\mathbf{x}_i\E\brac{\mathbf{x}_v}}=\frac{d}{n}\Paren{1+\frac{1}{2}\eps\cdot\Paren{p-(1-p)}}\\
    &=\frac{d}{n}\Paren{1+\eps\gamma\mathbf{x}_i}\,.
\end{align*}

Therefore, the conditional distribution of $\mathrm{\textbf{deg}}_{\neq j}(i)$ given $\mathbf{x}_i$ is $\mathrm{Binomial}\Paren{n-2, \frac{d}{n}\Paren{1+\eps\gamma\mathbf{x}_i}}$, hence
\begin{align*}
    \E\Brac{\mathrm{\textbf{deg}}_{\neq j}(i)\given \mathbf{x}_i}&=(n-2)\cdot\frac{d}{n}\Paren{1+\eps\gamma\mathbf{x}_i}\,,
\end{align*}
and so
\begin{equation}
\label{eq:conditional-expectation-centered-degree-without-bounding}
    \E\Brac{\frac{1}{C}\Paren{\mathrm{\textbf{deg}}_{\neq j}(i)-d(1-2/n)}\given \mathbf{x}_i}=\frac{d(1-2/n)\cdot{\eps\gamma\mathbf{x}_i}}{C}\,.
\end{equation}

On the other hand, by the Cauchy-Schwarz inequality, we have
\begin{align*}
    \E&\Brac{\Abs{\hat{x}_i^{(j)}(\mathbf{G})-\frac{1}{C}\Paren{\mathrm{\textbf{deg}}_{\neq j}(i)-d(1-2/n)}}\given \mathbf{x}_i}\\
    &\leq \E\Brac{\frac{1}{C}\Abs{\mathrm{\textbf{deg}}_{\neq j}(i)-d(1-2/n)}\cdot\Bracbb{\Abs{\mathrm{\textbf{deg}}_{\neq j}(i)-d(1-2/n)}> C}\given \mathbf{x}_i}\\
    &\leq \frac{1}{C}\E\Brac{\Paren{\mathrm{\textbf{deg}}_{\neq j}(i)-d(1-2/n)}^2\given \mathbf{x}_i}^{1/2}\cdot\E\Brac{\Bracbb{\Abs{\mathrm{\textbf{deg}}_{\neq j}(i)-d(1-2/n)}> C}^2\given \mathbf{x}_i}^{1/2}\\
    &\leq \frac{1}{C}\E\Brac{\Paren{\mathrm{\textbf{deg}}_{\neq j}(i)-d(1-2/n)}^2\given \mathbf{x}_i}^{1/2}\cdot\bbP\Paren{\Abs{\mathrm{\textbf{deg}}_{\neq j}(i)-d(1-2/n)}> C\given \mathbf{x}_i}^{1/2}\,,
\end{align*}
and by Chebychev's inequality, we have
\begin{align*}
    \bbP\Paren{\Abs{\mathrm{\textbf{deg}}_{\neq j}(i)-d(1-2/n)}> C\given \mathbf{x}_i}
    &\leq \frac{\E\Brac{\Paren{\mathrm{\textbf{deg}}_{\neq j}(i)-d(1-2/n)}^2\given \mathbf{x}_i}}{C^2}\,,
\end{align*}
hence
\begin{equation}
\label{eq:degree-expectation-concentration-inequality}
    \E\Brac{\Abs{\hat{x}_i^{(j)}(\mathbf{G})-\frac{1}{C}\Paren{\mathrm{\textbf{deg}}_{\neq j}(i)-d(1-2/n)}}\given \mathbf{x}_i}\leq \frac{\E\Brac{\Paren{\mathrm{\textbf{deg}}_{\neq j}(i)-d(1-2/n)}^2\given \mathbf{x}_i}}{C^2}\,.
\end{equation}

Now notice that
\begin{align*}
    \E\Brac{\Paren{\mathrm{\textbf{deg}}_{\neq j}(i)-d(1-2/n)}^2\given \mathbf{x}_i}&=\E\Brac{\Paren{\Paren{\mathrm{\textbf{deg}}_{\neq j}(i)-d(1-2/n)\Paren{1+\eps\gamma\mathbf{x}_i} + \Paren{\eps\gamma\mathbf{x}_i}d(1-2/n) }}^2\given \mathbf{x}_i}\\
    &\stackrel{(\ast)}{\leq}\E\Brac{2\Paren{\mathrm{\textbf{deg}}_{\neq j}(i)-d(1-2/n)\Paren{1+\eps\gamma\mathbf{x}_i}}^2 + 2\Paren{\Paren{\eps\gamma\mathbf{x}_i}d(1-2/n) }^2\given \mathbf{x}_i}\\
    &\leq 2\E\Brac{\Paren{\mathrm{\textbf{deg}}_{\neq j}(i)-d(1-2/n)\Paren{1+\eps\gamma\mathbf{x}_i}}^2\given \mathbf{x}_i}+2d^2\epsilon^2\gamma^2\,,
\end{align*}
where $(\ast)$ is true because $(a+b)^2\leq 2a^2+2b^2$ for all $a,b\in\mathbb{R}$.

Now since the conditional distribution of $\mathrm{\textbf{deg}}_{\neq j}(i)$ given $\mathbf{x}_i$ is $\mathrm{Binomial}\Paren{n-2, \frac{d}{n}\Paren{1+\eps\gamma\mathbf{x}_i}}$, its conditional expectation is equal to $d(1-2/n)\Paren{1+\eps\gamma\mathbf{x}_i}$ and its conditional variance is equal to
\begin{align*}
    \E\Brac{\Paren{\mathrm{\textbf{deg}}_{\neq j}(i)-d(1-2/n)\Paren{1+\eps\gamma\mathbf{x}_i}}^2\given \mathbf{x}_i}
    &= (n-2)\Paren{\frac{d}{n}\Paren{1+\eps\gamma\mathbf{x}_i}}\cdot \Paren{1-\frac{d}{n}\Paren{1+\eps\gamma\mathbf{x}_i}}\\
    &\leq d(1+\eps\gamma)\leq 2d\,.
\end{align*}
Therefore,
\begin{align*}
    \E\Brac{\Paren{\mathrm{\textbf{deg}}_{\neq j}(i)-d(1-2/n)}^2\given \mathbf{x}_i}
    &\leq 4d+2d^2\epsilon^2\gamma^2\,,
\end{align*}
and hence from \eqref{eq:degree-expectation-concentration-inequality} we get
\begin{align*}
    \E&\Brac{\Abs{\hat{x}_i^{(j)}(\mathbf{G})-\frac{1}{C}\Paren{\mathrm{\textbf{deg}}_{\neq j}(i)-d(1-2/n)}}\given \mathbf{x}_i}\leq \frac{4d+2d^2\epsilon^2\gamma^2}{C^2}\,.
\end{align*}

Combining this with \eqref{eq:conditional-expectation-centered-degree-without-bounding} we get
\begin{align*}
    \E\Brac{\hat{x}_i^{(j)}(\mathbf{G})\given \mathbf{x}_i}&=\frac{d(1-2/n)\cdot\eps\gamma\mathbf{x}_i}{C}\pm O\Paren{\frac{d+d^2\epsilon^2\gamma^2}{C^2}}\\
    &=\frac{d(1-2/n)}{C}\Paren{\eps\gamma\mathbf{x}_i\pm O\Paren{\frac{1+d\epsilon^2\gamma^2}{C}}}\,.
\end{align*}

Let
$$C=\tilde{C}(1-2/n)\Paren{d\epsilon +  \frac{1}{\epsilon \mu'_\delta}}\,,$$
for some large enough constant $\tilde{C}\geq 1$ to be chosen later. We have
$$O\Paren{\frac{1}{C}}\leq O\Paren{\frac{1}{\tilde{C}(1-2/n)/(\eps\mu'_\delta)}}=O\Paren{\frac{\eps\mu'_\delta}{\tilde{C}}}\leq O\Paren{\frac{\eps\gamma}{\tilde{C}}}\,,$$
where the last inequality follows from $\gamma\geq\frac{\mu_\delta'}{2}$. Furthermore,
$$
O\Paren{\frac{d\epsilon^2\gamma^2}{C}}\leq O\Paren{\frac{d\epsilon^2\gamma^2}{\tilde{C}(1-2/n)\cdot d\epsilon}}\leq O\Paren{\frac{\epsilon\gamma^2}{\tilde{C}}}\leq O\Paren{\frac{\epsilon\gamma}{\tilde{C}}}\,.
$$
We conclude that
\begin{align*}
    \E\Brac{\hat{x}_i^{(j)}(\mathbf{G})\given \mathbf{x}_i}
    &=\frac{d(1-2/n)}{C}\Paren{\eps\gamma\mathbf{x}_i\pm O\Paren{\frac{\epsilon\gamma}{\tilde{C}}}}=\frac{d(1-2/n)\epsilon\gamma}{C}\Paren{\mathbf{x}_i\pm O\Paren{\frac{1}{\tilde{C}}}}\\
    &=\frac{d\epsilon\gamma}{\tilde{C}\Paren{d\epsilon +  \frac{1}{\epsilon \mu'_\delta}}}\Paren{\mathbf{x}_i\pm O\Paren{\frac{1}{\tilde{C}}}}=C''_\delta\gamma\Paren{\mathbf{x}_i\pm O\Paren{\frac{1}{\tilde{C}}}}\,,
\end{align*}
where
\begin{align*}
    C''_\delta&=\frac{d\epsilon^2}{\tilde{C}\Paren{d\epsilon^2 +  \frac{1}{ \mu'_\delta}}}=\frac{4(1+\delta)}{\tilde{C}\Paren{4(1+\delta)+  \frac{1}{ \mu'_\delta}}}=\frac{4(1+\delta)\mu'_\delta}{\tilde{C}\Paren{4(1+\delta)\mu'_\delta+  1}}\,.
\end{align*}
Notice how $C''_\delta$ depends only on $\delta\,.$

Finally,
\begin{align*}
    \E\Brac{\hat{\mathbf{X}}_{\mathrm{su}}(\mathbf{G})_{ij}\given \mathbf{x}_i,\mathbf{x}_j}
    &=\E\Brac{\hat{x}_i^{(j)}(\mathbf{G})\given \mathbf{x}_i}\cdot\E\Brac{\hat{x}_j^{(i)}(\mathbf{G})\given \mathbf{x}_j}\\
    &=C''_\delta\gamma\Paren{\mathbf{x}_i\pm O\Paren{\frac{1}{\tilde{C}}}}\cdot C''_\delta\gamma\Paren{\mathbf{x}_j\pm O\Paren{\frac{1}{\tilde{C}}}}\\
    &=(C''_\delta\gamma)^2\Paren{\mathbf{x}_i\mathbf{x}_j\pm O\Paren{\frac{1}{\tilde{C}}}}\,,
\end{align*}
and so
\begin{align*}
    \E\Brac{\hat{\mathbf{X}}_{\mathrm{su}}(\mathbf{G})_{ij}\given \mathbf{x}_i=\mathbf{x}_j} - \E\Brac{\hat{\mathbf{X}}_{\mathrm{su}}(\mathbf{G})_{ij}\given \mathbf{x}_i\neq\mathbf{x}_j}
    &=(C''_\delta\gamma)^2\Paren{2\pm O\Paren{\frac{1}{\tilde{C}}}}\,.
\end{align*}
By choosing $\tilde{C}$ to be an absolute constant which is large enough, we get
$$\E\Brac{\hat{\mathbf{X}}_{\mathrm{su}}(\mathbf{G})_{ij}\given \mathbf{x}_i=\mathbf{x}_j} - \E\Brac{\hat{\mathbf{X}}_{\mathrm{su}}(\mathbf{G})_{ij}\given \mathbf{x}_i\neq\mathbf{x}_j}\geq (C''_\delta\gamma)^2\geq \frac{(C''_\delta\mu_\delta')^2}{4}\,,$$
where the last inequality is true because we assume that $\gamma\geq \frac{\mu_\delta'}{2}$.
\end{proof}

Now we can leverage
\cref{lemma:weak-recovery-sufficiently-balanced} and \cref{lemma:weak-recovery-sufficiently-unbalanced} in order to prove \cref{theorem:weak-recovery-sbm}.

\begin{proof}[Proof of \cref{theorem:weak-recovery-sbm}]
Let $\gamma,\mu_\delta',C_\delta',C_\delta''$ be as in \cref{lemma:weak-recovery-sufficiently-balanced} and \cref{lemma:weak-recovery-sufficiently-unbalanced}.

We will first distinguish between the sufficiently balanced and sufficiently unbalanced cases by counting the number of edges which are incident to a random sublinear (but sufficiently high) set of vertices. Let $m=\lceil n^{3/4}\rceil$ and let $\mathbf{I}$ be a random subset of $[n]$ of size $m$.

Let $\mathrm{deg}(\mathbf{I})$ be the number of edges in $\mathbf{G}$ from $\mathbf{I}$ to $[n]\setminus\mathbf{I}$. We have
\begin{align*}
    \E\Brac{\mathrm{deg}(\mathbf{I})\given \mathbf{I}}
    &=\sum_{i\in\mathbf{I},j\in[n]\setminus\mathbf{I}}\E\Brac{\Bracbb{\{i,j\}\in \mathbf{G}}}=\sum_{i\in\mathbf{I},j\in[n]\setminus\mathbf{I}}\E\Brac{\bbP\Paren{\{i,j\}\in \mathbf{G}\given\mathbf{x}_i,\mathbf{x}_j}}\\
    &=\sum_{i\in\mathbf{I},j\in[n]\setminus\mathbf{I}}\E\Brac{\frac{d}{n}\Paren{1+\frac12\epsilon\mathbf{x}_i\mathbf{x}_j}}=\frac{d}{n}\sum_{i\in\mathbf{I},j\in[n]\setminus\mathbf{I}}\Paren{1+\frac12\eps\E\Brac{\mathbf{x}_i\mathbf{x}_j}}\\
    &=\frac{d}{n}\sum_{i\in\mathbf{I},j\in[n]\setminus\mathbf{I}}\Paren{1+\frac12\eps\Paren{\bbP\Paren{\mathbf{x}_i=\mathbf{x}_j}-\bbP\Paren{\mathbf{x}_i\neq\mathbf{x}_j}}}\\
    &=\frac{d}{n}\sum_{i\in\mathbf{I},j\in[n]\setminus\mathbf{I}}\Paren{1+\frac12\eps\Paren{
    p^2+(1-p)^2-2p(1-p)}}\\
    &=\frac{dm(n-m)}{n}\Paren{1+\frac12\eps\Paren{2p-1}^2}\\
    &=\frac{dm(n-m)}{n}\Paren{1+2\eps\gamma^2}=(1\pm o(1))dn^{3/4}\Paren{1+2\eps\gamma^2}\,.\\
\end{align*}

So the algorithm $\hat{\mathbf{X}}$ is defined as follows:
\begin{itemize}
    \item If $\mathrm{deg}(\mathbf{I})\geq dn^{3/4}\Paren{1+2\eps\Paren{\frac{3\mu_\delta'}{4}}^2}\,$ we apply the algorithm $\hat{\mathbf{X}}_{\mathrm{su}}$ on the subgraph $\mathbf{G}([n]\setminus\mathbf{I})$ of $\mathbf{G}$ induced on $n\setminus \mathbf{I}$ and define
    \begin{align*}
        \hat{\mathbf{X}}(\mathbf{G})_{ij}=
        \begin{cases}
        \hat{\mathbf{X}}_{\mathrm{su}}(\mathbf{G}([n]\setminus\mathbf{I}))_{ij}\quad&\text{if }i,j\in[n]\setminus\mathbf{I}\,,\\
        0\quad&\text{otherwise}.\\
        \end{cases}
    \end{align*}
    \item If $\mathrm{deg}(\mathbf{I})< dn^{3/4}\Paren{1+2\eps\Paren{\frac{3\mu_\delta'}{4}}^2}\,$ we apply the algorithm $\hat{\mathbf{X}}_{\mathrm{sb}}$ on the subgraph $\mathbf{G}([n]\setminus\mathbf{I})$ of $\mathbf{G}$ induced on $n\setminus \mathbf{I}$ and define
    \begin{align*}
        \hat{\mathbf{X}}(\mathbf{G})_{ij}=
        \begin{cases}
        \hat{\mathbf{X}}_{\mathrm{sb}}(\mathbf{G}([n]\setminus\mathbf{I}))_{ij}\quad&\text{if }i,j\in[n]\setminus\mathbf{I}\,,\\
        0\quad&\text{otherwise}.\\
        \end{cases}
    \end{align*}
\end{itemize}

By standard concentration inequalities, we can show that:
\begin{itemize}
    \item If $\gamma<\frac{\mu'_\delta}{2}\,,$ then with probability $1-o(1)$ we have $\mathrm{deg}(\mathbf{I})< dn^{3/4}\Paren{1+2\eps\Paren{\frac{3\mu_\delta'}{4}}^2}\,$ and so we apply the algorithm $\hat{\mathbf{X}}_{\mathrm{sb}}$ which will succeed in achieving pair-wise weak recovery according to \cref{lemma:weak-recovery-sufficiently-balanced}, assuming $i,j\in[I]$.
    \item If $\gamma\geq \frac{\mu'_\delta}{2}\,,$ then with probability $1-o(1)$ we have $\mathrm{deg}(\mathbf{I})\geq dn^{3/4}\Paren{1+2\eps\Paren{\frac{3\mu_\delta'}{4}}^2}\,$ and so we apply the algorithm $\hat{\mathbf{X}}_{\mathrm{ub}}$ which will succeed in achieving pair-wise weak recovery according to \cref{lemma:weak-recovery-sufficiently-unbalanced}, assuming $i,j\in[I]$.
    \item If $\frac{\mu'_\delta}{2}<\gamma<\mu'_\delta\,,$ then it follows from \cref{lemma:weak-recovery-sufficiently-balanced} and \cref{lemma:weak-recovery-sufficiently-unbalanced} that it does not matter which algorithm we apply because both of them achieve pair-wise weak recovery, assuming $i,j\in[I]$.
\end{itemize}

Now for any $i,j\in[n]\,,$ the probability of picking any of them in $\mathbf{I}$ is vanishingly small. We conclude that the algorithm $\hat{\mathbf{X}}$ satisfies the guarantees sought in \cref{theorem:weak-recovery-sbm}.

\end{proof}

\paragraph{Deferred proof of \cref{section:algorithm}}

First we prove \cref{fact:balanced-instance}.
\begin{proof}[Proof of \cref{fact:balanced-instance}]
Let $p\in [k]$ be fixed. By Chernoff's bound, with probability at least $1-n^{20}$, $\Paren{1-\sqrt{\frac{400k\log n}{n}}}\frac{n}{k}\leq \Snorm{c_p(\mathbf z)}\leq \Paren{1+\sqrt{\frac{400k\log n}{n}}}\frac{n}{k}\,,$ hence the property follows by a union bound.
\end{proof}

Next we prove \cref{fact:concentration-edges}.

\begin{proof}
For each $\ell\in [t]\,,$ define
$$\E\Brac{\hat{\mathbf{X}}(\mathbf G_{\ell})_\ij \suchthat \mathbf{f}_\ell(\mathbf{z})_i= \mathbf{f}_\ell(\mathbf{z})_j}=:C^*_\ell$$
and let $$C^*=\sum_{\ell \in [t]}C_\ell\,.$$
By independence of the observations, the inequality of \cref{fact:concentration-edges} for the case case $\mathbf{z}_i= \mathbf{z}_j$  follows with an application of Hoeffding's inequality.

The case $\mathbf{z}_i\neq \mathbf{z}_j$ needs a bit more work. By \cref{theorem:weak-recovery-sbm}, for each $\ell\in[t]$, we have
$$\E\Brac{\hat{\mathbf{X}}(\mathbf G_{\ell})_\ij\given \mathbf{f}_\ell(\mathbf{z})_i= \mathbf{f}_\ell(\mathbf{z})_j}=C^*_\ell\,,$$
and
$$\E\Brac{\hat{\mathbf{X}}(\mathbf G_{\ell})_\ij\given \mathbf{f}_\ell(\mathbf{z})_i\neq \mathbf{f}_\ell(\mathbf{z})_j}\leq C^*_\ell -  C_{\ell}\,.$$

Therefore,
\begin{align*}
\E\Brac{\hat{\mathbf{X}}(\mathbf G_{\ell})_\ij\given \mathbf{z}_i\neq \mathbf{z}_j}
=&\E\Brac{\hat{\mathbf{X}}(\mathbf G_{\ell})_\ij\given \mathbf{f}_\ell(\mathbf{z})_i= \mathbf{f}_\ell(\mathbf{z})_j}\cdot\bbP\Paren{\mathbf{f}_\ell(\mathbf{z})_i= \mathbf{f}_\ell(\mathbf{z})_j\given \mathbf{z}_i\neq \mathbf{z}_j}\\
&+\E\Brac{\hat{\mathbf{X}}(\mathbf G_{\ell})_\ij\given \mathbf{f}_\ell(\mathbf{z})_i\neq \mathbf{f}_\ell(\mathbf{z})_j} \cdot\bbP\Paren{\mathbf{f}_\ell(\mathbf{z})_i\neq \mathbf{f}_\ell(\mathbf{z})_j\given \mathbf{z}_i\neq \mathbf{z}_j}\\
\leq& C^*_\ell - C_\ell\cdot\frac{1}{2}\,,
\end{align*}
and so
$$\E\Brac{\sum_{\ell \in [t]}\hat{\mathbf{X}}(\mathbf G_{\ell})_\ij\suchthat \mathbf{z}_i\neq \mathbf{z}_j}\leq \sum_{\ell\in[t]} \Paren{C^*_\ell - C_\ell\cdot\frac{1}{2}} = C^* - \frac{\Cbar\cdot t}{2}\,.$$
The inequality of \cref{fact:concentration-edges} for case $\mathbf{z}_i\neq \mathbf{z}_j$ follows with another application of Hoeffding's inequality.
\end{proof}

Now we prove \cref{lemma:sucess-when-chosing-good-representative}.
\begin{proof}[Proof of \cref{lemma:sucess-when-chosing-good-representative}]
Let $A_i$ be a $(p,q)$-representative and  $A_j$ be a $(p',q)$-representative, for $p\,,p'\in[k]\,.$
It suffices to show that if $p=p'$ then $\Snorm{A_i-A_j}\leq n/k$ and otherwise $\Snorm{A_i-A_j}> n/k\,.$  Then no $q$-representative index remains unassigned at the end of step $1$.
By the reverse triangle inequality
\begin{align*}
    &\Abs{\Norm{A_i-A_j}-\Norm{c_p(z) -c_{p'}(z)}}\\
    &\leq \Norm{A_i-A_j-c_p(z)+c_{p'}(z)} \\
    &\leq \Norm{A_i-c_p(z)}+\Norm{A_j-c_{p'}(z)} \\
    &\leq 2n\cdot e^{-q\cdot\Cbar^2\cdot t}\,.
\end{align*}
For $p=p'$ we have $\Norm{c_p(z) -c_{p'}(z)}=0$ and by choice of $t\,,$ the first inequality follows.
For $p\neq p'$ we have $\Norm{c_p(z) -c_{p'}(z)}\geq \frac{n}{k}(2-o(1))$ since $z$ is balanced and the second inequality follows again by choice of $t$.
\end{proof}

\end{document}